\documentclass[12pt,a4paper,eqrno]{article}
\usepackage{a4wide}
\usepackage[latin1,utf8]{inputenc}
\usepackage[english]{babel}
\usepackage[T1]{fontenc}

\usepackage[margin=1cm,labelfont=bf]{caption}
\usepackage{graphicx}
\usepackage{amsmath,amssymb,amsthm,mathtools,dsfont,mathrsfs}
\usepackage{enumitem} 
\usepackage{relsize}
\usepackage{dsfont}
\usepackage[dvipsnames]{xcolor}
\usepackage{pgfplots}
\pgfplotsset{compat=1.17}
\usepackage{tikz}
\usepackage{tikz-cd}
\usetikzlibrary{arrows.meta}
\usepackage[top=3cm,bottom=3cm,left=2cm,right=2cm,marginparwidth=0cm]{geometry}
\usepackage[margin=1cm]{caption}
\usepackage[colorlinks=true]{hyperref}
 \hypersetup{
     colorlinks,
     citecolor={PineGreen},
     linkcolor={MidnightBlue},
     urlcolor={RedViolet}
 }
\usepackage{longtable}
\usepackage{algorithm}
\usepackage{algpseudocode}
\usepackage{float}
\makeatletter
\newenvironment{breakablealgorithm}
  {
   \begin{center}
     \refstepcounter{algorithm}
     \hrule height.5pt depth0pt \kern-3pt
     \renewcommand{\caption}[2][\relax]{
       {\raggedright\textbf{\ALG@name~\thealgorithm} ##2\par}%
       \ifx\relax##1\relax 
         \addcontentsline{loa}{algorithm}{\protect\numberline{\thealgorithm}##2}%
       \else 
         \addcontentsline{loa}{algorithm}{\protect\numberline{\thealgorithm}##1}%
       \fi
       \kern2pt\hrule\kern2pt
     }
  }{
     \kern2pt\hrule\relax
   \end{center}
  }
\makeatother
\usepackage{setspace}
\usepackage[scr=boondox] 
           {mathalpha}

\newcommand{\N}{\mathbb{N}}

\newcommand{\R}{\mathbb{R}}
\newcommand{\RR}{\mathbb{R}}

\newcommand{\Gcal}{\mathcal{G}}

\renewcommand{\epsilon}{\varepsilon}

\DeclareMathOperator{\rank}{rank}
\DeclareMathOperator{\soft}{soft}
\DeclareMathOperator{\pa}{pa}
\DeclareMathOperator{\ch}{ch}
\DeclareMathOperator{\an}{an}
\DeclareMathOperator{\de}{de}

\newcommand{\Mns}{\mathord{\begin{tikzpicture}[baseline=0ex, line width=2, scale=0.15] \draw (0,0.6) -- (2.2,0.6); \end{tikzpicture}}}

\mathtoolsset{showonlyrefs=true}

\theoremstyle{definition}
\newtheorem{theorem}{Theorem}[section]
\newtheorem{proposition}[theorem]{Proposition}
\newtheorem{corollary}[theorem]{Corollary}

\newtheorem{definition}[theorem]{Definition}
\newtheorem{assumption}[theorem]{Assumption}
\newtheorem{problem}[theorem]{Problem}

\newtheorem{remark}[theorem]{Remark}

\newtheorem{examplex}[theorem]{Example}
\newenvironment{example}
{\pushQED{\qed}\examplex}
{\popQED\endexamplex}

\usepackage[parfill]{parskip}

\usepackage{todonotes}

\title{Linear causal disentanglement 
\\ via higher-order cumulants}

\usepackage{authblk}
\author[1]{Paula Leyes Carreno}
\author[2]{Chiara Meroni
}
\author[3]{Anna Seigal}
\affil[ ]{\footnotesize \url{pleyescarreno@college.harvard.edu}, \, \url{chiara.meroni@eth-its.ethz.ch}, \, \url{aseigal@seas.harvard.edu}}
\affil[1,3]{\normalsize Harvard University, Cambridge, MA, United States}
\affil[2]{ETH-ITS, Zurich, Switzerland}

\date{}

\begin{document}
\maketitle

\begin{abstract}
Linear causal disentanglement is a recent method in causal representation learning to describe a collection of observed variables via latent variables with causal dependencies between them. It can be viewed as a generalization of both independent component analysis and linear structural equation models. We study the identifiability of linear causal disentanglement, assuming access to data under multiple contexts, each given by an intervention on a latent variable. We show that one perfect intervention on each latent variable is sufficient and in the worst case necessary to recover parameters under perfect interventions, generalizing previous work to allow more latent than observed variables. We give a constructive proof that computes parameters via a coupled tensor decomposition. For soft interventions, we find the equivalence class of latent graphs and parameters that are consistent with observed data, via the study of a system of polynomial equations. Our results hold assuming the existence of non-zero higher-order cumulants, which implies non-Gaussianity of variables. 
\end{abstract}

\vspace{-2ex}
\paragraph{MSC classes.} 
13P15, 
15A69, 
62H22, 
62R01, 
68Q32. 

\paragraph{Acknowledgements.} 
We thank Kexin Wang for helpful discussions.
CM was supported by Dr. Max R\"ossler, the Walter Haefner Foundation and the ETH Z\"urich Foundation. 
AS was partially supported by the NSF (DMR-2011754).

\section{Introduction}

A key challenge of data science is to find useful and interpretable ways to model complex data, such as those collected from a biological experiment or a physical system. 
In this paper, we study \emph{linear causal disentanglement} (LCD), a framework to model such data. LCD generalizes two 20th century data analysis models: \emph{independent component analysis} (ICA)~\cite{jutten1987calcul,ComonICA,comon2010handbook} and \emph{linear structural equation models} (LSEMs)~\cite{bollen1989structural,sullivant2018algebraic}. Before defining it, we briefly recall these older models.

ICA is a blind source separation method that expresses observed variables $X = (X_1, \ldots, X_p)$ as a linear mixture
\begin{equation}
    \label{eqn:ica}
    X = A \epsilon,
\end{equation} 
where $A \in \RR^{p \times q}$ is a mixing matrix and $\epsilon = (\epsilon_1, \ldots, \epsilon_q)$ is a vector of independent latent variables. ICA has been used in applications including brain dynamics~\cite{jung2001imaging} and astrophysics~\cite{berne2007analysis}. 
LSEMs are another linear model to describe collections of variables. They model variables $Z = (Z_1, \ldots,Z_q)$ as
\begin{equation}\label{eq:lsem}
    Z=\Lambda Z+\epsilon,
\end{equation}
where $\Lambda \in \RR^{q \times q}$ is a matrix whose entry $\lambda_{i,j}$ encodes the dependence of $Z_i$ on $Z_j$ and $\epsilon$ is a vector of noise variables, often assumed to be independent. 
The variables are typically assumed to relate via the recursive structure of a directed acyclic graph (DAG); that is, 
fixing a DAG $\Gcal$ on nodes $[q] = \{ 1, \ldots, q\}$, with directed edges denoted $j \to i$, we have 
\begin{equation}
\label{eqn:DAG}
    \lambda_{i,j} \neq 0 
    \qquad \iff \qquad (j \to i) \in \Gcal.
\end{equation}
Equation~\eqref{eq:lsem} can be re-written as $Z = (I - \Lambda)^{-1} \epsilon$, 
where acyclicity of $\Gcal$ ensures that the matrix $I - \Lambda$ is invertible. This places LSEMs in the context of ICA, since the variables $Z$ are a linear mixing of independent latent variables~\cite{shimizu2014lingam}.
LSEMs appear in applications including epidemiology~\cite{sanchez2005structural} and causal inference~\cite{pearl2000models}.
In causal inference, the quantity $\lambda_{i,j}$ is interpreted as the causal effect of $Z_j$ on $Z_i$. 

The idea of linear causal disentanglement~\cite{SSBU23:LinearCausalDisentanglementInterventions} 
 is that the assumptions of ICA and LSEMs may be too strict: interpretable latent variables may not be independent, and variables that relate via a graph may not have been directly measured. 
To get around this, LCD is defined as follows. 
As in ICA, we observe variables $X = (X_1, \ldots, X_p)$ that are a linear 
mixing of latent variables. However, unlike ICA, the latent variables are not independent, instead they follow the structure of an LSEM; that is,
\begin{equation}\label{eq:XandZ}
X = FZ , \qquad \text{where} \qquad  Z=\Lambda Z+\epsilon, 
\end{equation}
for $F \in \RR^{p \times q}$ a linear transformation, $\Lambda$ a matrix that encodes causal dependencies among the latent variables $Z = (Z_1, \ldots, Z_q)$, and $\epsilon = (\epsilon_1, \ldots, \epsilon_q)$ a vector of independent noise variables. As often the case in ICA and LSEMs, variables $\epsilon$ are assumed to be mean-centered.
LCD specializes to ICA when $\Lambda$ is the zero matrix (i.e. when $\Gcal$ is the empty graph) and to an LSEM when $F = I$.

LCD falls into the setting of \emph{causal representation learning}~\cite{scholkopf2021toward}, an area of machine learning that aims to describe and explain the structure of a complex system by learning variables together with the causal dependencies among them. The idea is that 
learned latent representations of data~\cite{bengio2014representation} can be difficult to interpret and analyze, and may not generalize well, but that they improve by using latent representations with causal structure~\cite{yang2021causalvae}.
Central to interpretability and downstream analysis is the identifiability of a representation. The LCD model~\eqref{eq:XandZ} is identifiable if the mixing matrix $F$ and matrix of dependencies  $\Lambda$, and therefore also the latent DAG $\Gcal$, can be recovered uniquely (or up to a well-described set of possibilities) from observations of $X$. 

In this paper, we study the identifiability of LCD and develop algorithms to recover the parameters $F$ and $\Lambda$ using 
 tensor decomposition of higher-order cumulants. 
Higher-order cumulants have been used to recover parameters in both ICA and LSEMs~\cite{shimizu2014lingam,wang2020high,ComonICA,de2007fourth,wang2024identifiability}. We build on these insights to use it for LCD. 
For ICA and LSEMs, parameters can be recovered from tensor decomposition of a single higher-order cumulant. 
For LCD one tensor decomposition no longer suffices to recover parameters and we will instead use a coupled tensor decomposition. 
Identifiability of LCD from covariance matrices (that is, second-order cumulants) was studied in~\cite{SSBU23:LinearCausalDisentanglementInterventions}. Our results extend these insights to identifiability via higher-order cumulants.

 \paragraph{The setup.}
Our goal in this paper is to use observations of $X$ to recover the parameters $F$ and $\Lambda$ in an LCD model~\eqref{eq:XandZ}.
We assume access to observations of $X$ under multiple contexts.  
The contexts differ from an observational context by an intervention.
Interventions appear in biological applications such as~\cite{sachs2005causal,dixit2016perturb,mhmpvb2016,Triantafillou2017,xi2024propensity}. 
An intervention at a variable affects the downstream variables but not those that are upstream. It thus enables one to learn the direction of a causal dependency between two variables.
We study multiple contexts for two reasons: inferring causal dependencies in general necessitates interventions and 
one context is insufficient for recovery of parameters in the model.
We consider two types of interventions.

\begin{definition}\label{def:types_interventions}
Let variables $Z_i$ relate via a linear structural equation model. 
A {\em soft} (resp. {\em perfect}) {\em intervention} at $Z_i$ changes (resp. zeros out) all non-zero weights $\lambda_{i,j}$ and changes the error distribution $\epsilon_i$. 
\end{definition}

A third widely-studied type of intervention is a {\em do-intervention}, which sets a variable to a deterministic value.
We focus on soft interventions and perfect interventions, so that we do not assume access to a fixed value of an unobserved variable. 
For related results for do-interventions, see \cite{yang2021causalvae,AMWB23:InterventionalCausalReprLearning}.

We denote the set of contexts by $K$. 
Each context $k \in K$ is assumed to be an intervention at a single latent variable, 
as in~\cite{SSBU23:LinearCausalDisentanglementInterventions}.
The target of each intervention is unknown: context $k$ is an intervention on $Z_{i_k}$ for some $i_k \in [q]$.
The observational setting, in which no variable is intervened on, is indexed by $k =0$ and assumed to be known.
The intervention changes the latent LSEM but not the mixing map $F$. 
Under context $k$, we denote the matrix of causal effects by $\Lambda^{(k)}$, the latent variables by $Z^{(k)}$, and the error distributions by $\epsilon^{(k)}$.
Error distributions $\epsilon^{(k)}$ and $\epsilon^{(0)}$ agree, except at the $i_k$-th entry.
From Definition~\ref{def:types_interventions}, we see that a perfect intervention sets the $i_k$-th row of $\Lambda^{(k)}$ to zero while a soft intervention satisfies $\lambda^{(k)}_{i_k,j} \neq \lambda^{(0)}_{i_k,j}$ whenever $\lambda^{(0)}_{i_k,j} \neq 0$, i.e. for all $j$ with edge $j \to i_k$ present in $\Gcal$. 
Our setup can now be summarized as follows.

Fix $p \geq 2$ observed variables.
We observe distributions $X^{(k)}$ on $\RR^p$ for $k \in K \cup \{ 0 \}$ of the form 
\begin{equation}\label{eq:our_setup}
X^{(k)} = F Z^{(k)}, \qquad \text{where} \qquad Z^{(k)} = \Lambda^{(k)} Z^{(k)} + \epsilon^{(k)} ,
\end{equation}
for $Z^{(k)}$ some distributions on $\RR^q$ where $q \geq 2$ is the  number of latent variables. The variables $Z^{(0)}$ on $\RR^q$ follow a linear structural equation model on an unknown DAG $\Gcal$ on $q$ nodes,
and $Z^{(k)}$ relates to $Z^{(0)}$ via a single-node perfect or soft intervention with unknown target. See Figure \ref{fig:cartoon} for a cartoon of our setup. We make the following genericity assumptions.

\begin{assumption}\label{assumption:main} 
~
\begin{enumerate}
    \item[(a)] All noise variables $\epsilon_i^{(k)}$ are non-Gaussian. 
    \item[(b)] Matrix $F \in \RR^{p \times q}$ is unknown and generic; matrices $\Lambda^{(k)} \in \RR^{q \times q}$, $k \in K \cup \{0\}$ are unknown with generic non-zero entries.
    \item[(c)] For all contexts $k \in K$ there exists a large enough $d$  ($d\geq 3(q-1)$ is sufficient) such that the $d$-th order cumulant of $\epsilon^{(i_k)}_{i_k}$ is not $0$ or $\pm1$.
\end{enumerate}
\end{assumption}
\begin{figure}[th]
    \centering
    \fbox{
    \begin{tikzpicture}
    \node at (-3,2) {latent variables:};
    \node at (-3,-1) {observed variables:};
    \node at (-3,0.5) {mixing map:};
    \begin{scope}[every node/.style={circle,thick,draw}]
        \node (Z1) at (4,2) {$Z_1$};
        \node (Z2) at (2,2) {$Z_2$}; 
        \node (Z3) at (0,2) {$Z_3$}; 
        \node (X2) at (1,-1) {$X_2$};
        \node (X1) at (3,-1) {$X_1$};
    \end{scope}
    \begin{scope}[>={Stealth[PineGreen]},
                  every node/.style={fill=white,circle},
                  every edge/.style={draw=PineGreen,very thick}]
        \path [->] (Z3) edge (Z2);
        \path [->] (Z3) edge[bend right = -30] (Z1);
    \end{scope}
    \node[PineGreen] at (1,1.73) {$\lambda_{2,3}$};
    \node[PineGreen] at (0.7,2.75) {$\lambda_{1,3}$};
    \begin{scope}[>={Stealth[RedViolet]},
                  every node/.style={fill=white,circle},
                  every edge/.style={draw=RedViolet,very thick}]
        \path [->] (Z3) edge (X2);
        \path [->] (Z3) edge (X1);
        \path [->] (Z2) edge (X2);
        \path [->] (Z2) edge (X1);
        \path [->] (Z1) edge (X2);
        \path [->] (Z1) edge (X1);
    \end{scope}
    \begin{scope}[every node/.style={RedViolet,scale=0.85,thick}]
        \node at (0.3,0) {$f_{2,3}$};
        \node at (0.8,0.9) {$f_{1,3}$};
        \node at (1.9,0.8) {$f_{2,2}$};
        \node at (2.55,1.4) {$f_{1,2}$};
        \node at (3,0.6) {$f_{2,1}$};
        \node at (3.6,-0.2) {$f_{1,1}$};
    \end{scope}
    \end{tikzpicture}
    }
    \caption{A cartoon of the setup for $p=2$ observed variables and $q=3$ latent variables.}
    \label{fig:cartoon}
\end{figure}
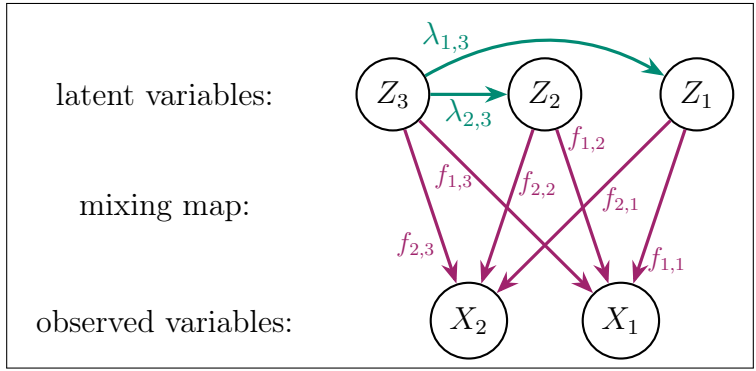

\begin{problem}
\label{prob:main}
    In the setup \eqref{eq:our_setup} under Assumption~\ref{assumption:main}, learn the number of latent variables $q$, the latent DAG $\Gcal$, the mixing matrix $F$ and the matrices of dependencies $\{ \Lambda^{(k)} \,|\, k \in K \cup \{ 0 \} \}$.
\end{problem}

We can rearrange~\eqref{eq:our_setup} to 
write variables $X^{(k)}$ as a linear mixture of independent latent variables
\[ X^{(k)} = F (I - \Lambda^{(k)})^{-1} \epsilon^{(k)}. \]
This relates LCD to ICA. 
Just as for ICA, we have the following non-identifiability.

\begin{remark}[Benign non-identifiability]
\label{rmk:benign}
    Uniqueness of $F$ and $\Lambda^{(k)}$ is impossible in LCD, since one can rescale or reorder the latent variables without affecting membership in the model. 
    That is, 
    for a diagonal matrix $D \in \RR^{q \times q}$ and a permutation matrix $P \in \RR^{q \times q}$,
    setting
    \begin{equation}
    \label{eqn:equiv}
        \widetilde{F} = FM, \quad \widetilde{\Lambda}^{(k)} = M^{-1} \Lambda^{(k)} M, \quad \widetilde{\epsilon}^{(k)} = M^{-1} \epsilon^{(k)}, \quad \text{where } M = DP,
    \end{equation} we have
    \begin{equation}
        \label{eqn:benign}
        \widetilde{F} (I - \widetilde{\Lambda}^{(k)} )^{-1} \widetilde{\epsilon}^{(k)} = F (I - \Lambda^{(k)})^{-1} \epsilon^{(k)}. 
    \end{equation} 
    Hence such rescaling and reordering does not affect $X^{(k)}$. Such transformations do not change the latent graph $\Gcal$ except by a relabelling of its nodes under the permutation $P$. When we discuss identifiability, we mean uniqueness up to the benign rescaling and reordering transformations in~\eqref{eqn:equiv}.  
    Given multiple contexts $k \in K \cup \{ 0 \}$, the scaling and ordering  transformations $D$ and $P$ are the same for all $k$. 
\end{remark}

\paragraph{Main results.} 

We find the perfect interventions needed for identifiability of LCD. 

\begin{theorem}\label{thm_intro_perfect}
Consider LCD under Assumption~\ref{assumption:main} with perfect interventions. Then one perfect intervention on each latent node is sufficient and, in the worst case, necessary to recover the latent DAG $\Gcal$ and the parameters $F$ and $\Lambda^{(k)}$ from observations of $X^{(k)}$. 
\end{theorem}

For $p$ observed variables and $q$ latent variables, Theorem~\ref{thm_intro_perfect} says that we need $q$ interventions for identifiability of LCD. We do not impose the injectivity of the mixing map $F : \RR^q \to \RR^p$; the pair $(p,q)$ can take any values provided $p,q \geq 2$.
Our proof is constructive: we carry out a coupled tensor decomposition of higher-order cumulants of the distributions $X^{(k)}$, and compare the factors recovered to learn the parameters. This extends~\cite{shimizu2014lingam,wang2020high} from observed to latent causal variables, and extends~\cite{ComonICA,de2007fourth,wang2024identifiability} from independent to dependent latent variables.
It relates to~\cite{eberhardt2005number}, which says that $q-1$ interventions are sufficient and in the worst case necessary to recover a DAG on $q$ observed variables.
It builds on~\cite[Theorem 1]{SSBU23:LinearCausalDisentanglementInterventions}, which says that one intervention on each latent node is sufficient and in the worst case necessary when the mixing $F$ is injective.
When the mixing map is injective, Theorem~\ref{thm_intro_perfect} is weaker than~\cite[Theorem 1]{SSBU23:LinearCausalDisentanglementInterventions}, since it requires non-Gaussian errors.
 When $F$ is not injective non-Gaussianity is necessary for identifiability, see Proposition \ref{prop:cov_not_enough}.

We present two algorithms for the recovery of the model parameters using $q$ perfect interventions. The first algorithm can be used for any $(p,q)$. It takes as input a tuple of $q+1$ cumulants, and returns the parameters $F$ and $\Lambda^{(k)}$. 
The second algorithm applies to the setting $q\leq p$. Here Moore-Penrose pseudo-inverses can be used to simplify the recovery. 
We illustrate the performance of the algorithms in Figures \ref{fig:recovery_F} and \ref{fig:graph}.
Both are implemented in \texttt{Python}, version $3.12.2$. The code is available at \url{https://github.com/paulaleyes14/linear-causal-disentanglement-via-cumulants}.

We now turn to soft interventions.
The transitive closure $\overline{\Gcal}$ of a DAG $\Gcal$ is the DAG with all edges $j \to i$ whenever $j \to \cdots \to i$ is a path in $\Gcal$.
We can recover the transitive closure $\overline{\Gcal}$ of a latent DAG $\Gcal$ in LCD from the second-order cumulants, see \cite[Theorem 1]{SSBU23:LinearCausalDisentanglementInterventions}.
We show that, if the errors are non-Gaussian, we can distinguish certain DAGs with the same transitive closure. 
We define the set of soft-compatible DAGs $\soft(\Gcal)$. It is a set of DAGs with the same transitive closure, which also satisfy additional compatibility conditions coming from ranks of matrices. Define the set of children of node $j$ by $\ch_\Gcal(j) = \{ i \,|\, (j\to i)\in \Gcal \}$ and the descendants by $\de_\Gcal(j) = \{ i \,|\, (j\to \cdots \to i)\in \Gcal \}$. Then,
\begin{multline*}
    \soft(\Gcal) = \Big\lbrace \Gcal'  \;\Big|\; \overline{\Gcal'} = \overline{\Gcal} \, \text{ and} \,  \rank [(I-\Lambda_{\Gcal})^{-1}]_{\mathscr{r}_j,\,\mathscr{c}_j} = \rank [(I-\Lambda_{\Gcal})^{-1}]_{\mathscr{r}_j,\,\mathscr{c}_j \cup \{ j \}} \hbox{ for all } j \in [q]  \Big\rbrace,
\end{multline*}
where $\mathscr{r}_j := \de_{\Gcal'}(j)\setminus\ch_{\Gcal'}(j)$, $\mathscr{c}_j := \ch_{\Gcal'}(j)$, and $\Lambda_{\Gcal}$ is a generic matrix of dependencies in an LSEM on DAG $\Gcal$, and $[M]_{\mathscr{r},\mathscr{c}}$ denotes the submatrix of $M$ with row indices in $\mathscr{r}$ and column indices in $\mathscr{c}$. See Definition \ref{def:softG} for more details.

\begin{theorem}\label{thm:soft1}
Consider LCD under Assumption~\ref{assumption:main} with soft interventions. Then one soft intervention on each latent node is sufficient and, in the worst case, necessary to recover the set of DAGs $\soft(\Gcal)$. Given $\Gcal' \in \soft(\Gcal)$, the set of parameters $F$ and $\Lambda^{(k)}$ that are compatible with the observations is a positive dimensional linear space.
\end{theorem}

The proof relies on the study of the solution space to a system of polynomial equations, encoding the conditions that parameters compatible with the observations must satisfy.
That space is linear and always positive dimensional, even if we allow multiple interventions on each latent node. This leads to a negative identifiability result, in the same spirit as \cite{locatello2019challenging}.
\begin{corollary}\label{coroll:identifiable_soft_more}
    Consider LCD under Assumption~\ref{assumption:main}. With any number of soft interventions, identifiability of all parameters in the model does not hold.
\end{corollary}

The non-Gaussianity assumption is required for the linear space of parameters in Theorem~\ref{thm:soft1}: with Gaussian errors, the space of parameters may be non-linear, see Proposition~\ref{prop:nonlinear_covariance}.

 \paragraph{Related work.}
Higher-order cumulants have been shown to lead to improved identifiability in related contexts. They extend principal component analysis, which requires an orthogonal transformation for identifiability, to ICA, which is identifiable for general linear mixings~\cite{comon2010handbook,ComonICA}. For LSEMs, they facilitate the recovery of a full DAG, rather than its Markov equivalence class~\cite{verma2022equivalence}, see~\cite{shimizu2014lingam,wang2020high}. They have been used to recover parameters in other latent variable models~\cite{anandkumar2014tensor}.

Identifiability of causal representation learning is an active area of study. 
It builds on work in the identifiability of representation learning~\cite{ahuja2021properties, zimmermann2022contrastive,khemakhem2020variational} and latent DAG models. These include work that imposes sparsity on the causal relations \cite{moran2022identifiable, ahuja2021properties, hyvärinen1999, NEURIPS2021_c0f6fb5d, silva2006, NEURIPS2019, xie2020, xie2022, liu2023identifying, NEURIPS2023_bdeab378, ZSGSSU23:IdentifiabilitySoftInterventions, lachapelle2024nonparametric, zhang2024causal} and latent variable models on discrete variables \cite{halpern2015anchored, kivva2021learning}. 
There are many works related to LCD, due in part to the many possible assumptions that one can make in a causal disentanglement model. These include the structure (polynomial, non-linear) of the maps involved \cite{NEURIPS2023_97fe251c, NEURIPS2023_8e5de4cb, varici2023scorebased, pmlr-v238-varici24a, liu2024identifiable, liu2024identifiable2, varıcı2024scorebased} and the choice of data generating process \cite{NEURIPS2023_4ef594af, bing2023invariance, kekic2024targeted, shamsaie2024disentangling}. In general, allowing more freedom on one side, implies more restrictions on the other side.

\paragraph{Outline.}
We cast LCD as the problem of aligning the outputs of a coupled tensor decomposition in Section \ref{sec:tensor_dec}.
We discuss the recovery of parameters for perfect and soft interventions in Section \ref{sec:recovery_interventions}. We prove our main results Theorem \ref{thm_intro_perfect} in Section \ref{subsec:perfect_int} and Theorem \ref{thm:soft1} in Section \ref{subsec:soft_int}. We discuss our algorithms in Section \ref{sec:algorithms} and future directions in Section~\ref{sec:outlook}. Appendix \ref{sec:Appendix} contains pseudo-code for our algorithms.

\section{Coupled tensor decomposition}\label{sec:tensor_dec}

 The cumulants are a sequence of tensors that encode a distribution~\cite{mccullagh2018tensor}. 
 The $d$-th cumulant 
 of a distribution $X$ on $\RR^p$
 is an order $d$ tensor, denoted by $\kappa_d(X)$, of format $p \times \cdots \times p$. The first and second order cumulants are the mean and covariance, respectively. 
Higher-order cumulants are those of order three and above.

We describe the higher-order cumulant tensors of distributions $X^{(k)}$ coming from LCD, as in~\eqref{eq:XandZ}, 
as $k$ ranges over contexts.
We study a coupled decomposition of these tensors. This will enable us to study the identifiability of LCD and to design tensor decompositions to recover parameters in the model. 
We first consider a single context.

\subsection{Decomposing cumulants}

Let $X$ be a distribution on $\RR^p$ and 
assume $X = A \epsilon$, where $\epsilon = (\epsilon_1, \ldots, \epsilon_q) $ is a vector of independent variables on $\RR^q$ and $A \in \RR^{p \times q}$ is a linear map, as in ICA~\eqref{eqn:ica}. 
Then the $d$-th cumulant of $X$ 
is the order $d$ tensor
\begin{equation}
\label{eqn:cumulant_decomp}
     \kappa_d(X) = \sum_{i = 1}^q \kappa_d(\epsilon_i) \mathbf{a}_i^{\otimes d},
\end{equation}
where the scalar $\kappa_d(\epsilon_i)$ is the $d$-th cumulant of variable $\epsilon_i$ and $\mathbf{a}_i$ is the $i$-th column of matrix $A$, as follows.
 The cumulants $\kappa_d(\epsilon)$
are order $d$ tensors of format $q \times \cdots \times q$.
Since the variables $\epsilon_i$ are independent, by assumption, their 
cross-cumulants vanish~\cite[Section 2.1]{mccullagh2018tensor}. 
Hence the tensor $\kappa_d(\epsilon)$ is diagonal: its entries vanish away from the $\kappa_d(\epsilon_1), \ldots, \kappa_d(\epsilon_q)$ on the main diagonal. 
A linear transformation of variables results in a multi-linear transformation of their cumulants. This gives the expression in~\eqref{eqn:cumulant_decomp}, which writes the cumulant as a sum of symmetric rank one tensors.

If $q \leq p$ then $\kappa_d(X)$ has a unique rank $q$ decomposition, whenever cumulants $\kappa_d(\epsilon_i)$ are all non-zero and the columns of $A$ are linearly independent, by~\cite{harshman70}. Hence the vectors $\mathbf{a}_i$ can be recovered uniquely, up to permutation and scaling. This extends to $q > p$, as follows.

\begin{proposition}
\label{prop:non_zero_coefficients}
Assume that no pair of columns of $A \in \RR^{p \times q}$ are collinear and that the $q$ entries of $\epsilon$ are independent. Then,
    for $d$ sufficiently large, 
    all columns $\mathbf{a}_i$ with
    $\kappa_d(\epsilon_i) \neq 0$
    can be uniquely recovered, up to permutation and scaling, from the $d$-th cumulant of $X = A\epsilon$. 
\end{proposition}

\begin{proof}
For $m \geq q-1$, the tensors $\mathbf{a}_1^{\otimes m}, \ldots, \mathbf{a}_q^{\otimes m}$ are linearly independent, by~\cite[Proposition 4.3.7.6]{landsberg2011tensors}, since no pair of columns $\mathbf{a}_i$ are collinear. Let $d \geq 3m \geq 3(q - 1)$ and consider
$ \kappa_{d}(X) = \sum_{i=1}^q \lambda_i \mathbf{a}_i^{\otimes d}$,
where $\lambda_i := \kappa_{d}(\epsilon_i)$. Consider its flattening of size $q^m \times q^m \times q^{d-2m}$. 
The decomposition of this flattened tensor is unique, by~\cite{harshman70}, since the vectors that appear in it are linearly independent. Hence the tensors $\mathbf{a}^{\otimes m}_i$ and $\mathbf{a}_i^{\otimes (d- 2m)}$, and thus also the vectors $\mathbf{a}_i$, can be uniquely recovered, up to permutation and scaling, for all indices $i$ with $\lambda_i \neq 0$.
\end{proof}

For a sufficiently generic matrix $A$, one can recover the vectors uniquely, up to permutation and scaling, from the above tensor decomposition provided $q$ is strictly less than the generic rank of an order $d$ tensor of format $p \times \cdots \times p$, by~\cite{chiantini2017generic}. The generic rank is usually $\lceil \frac{1}{p} { \binom{p + d -1}{d} } \rceil$, see~\cite[Theorem 3.2.2.4]{landsberg2011tensors}. Since for fixed $p$ and large $d$, $\frac{1}{p} { \binom{p + d -1}{d} } \sim d^{p-1}$, this result allows for larger $q$ relative to $d$ than the condition $d \geq 3(q-1)$ coming from Proposition~\ref{prop:non_zero_coefficients}.

\begin{corollary}
\label{cor:ica}
    Assume that the entries of $\epsilon$ are independent and non-Gaussian and that no pair of columns of $A$ are collinear. Then tensor decomposition of the cumulants of $X$ recovers the matrix $A$, up to permutation and scaling of its columns.
\end{corollary}

\begin{proof}
The cumulant sequence $(\kappa_d(\epsilon_i))_d$ has infinitely many non-zero terms, since $\epsilon_i$ is non-Gaussian~\cite{marcinkiewicz1939propriete}. Hence there are non-zero cumulants at high enough $d$
to satisfy the hypotheses of Proposition~\ref{prop:non_zero_coefficients}. 
This is an alternative proof of~\cite[Theorems 1(i) and 3(i)]{eriksson2004identifiability}.
\end{proof}

The impossibility of recovering the columns without scaling ambiguity comes from the fact that we can extract or insert a global scalar from the factor $\mathbf{a}_i^{\otimes d}$. We have $(\lambda \mathbf{a})^{\otimes d} = \lambda^d \mathbf{a}^{\otimes d}$, hence
\begin{equation}
    \label{eqn:lambdas_from_tensor}
    \kappa_d(X) = 
\begin{cases}
    \sum_{i = 1}^q \left( \sqrt[d]{\kappa_d(\epsilon_i) } \mathbf{a}_i \right)^{\otimes d} & d \text{ odd}  \\
    \sum_{i = 1}^q  \operatorname{sign}\left(\kappa_d(\epsilon_i)\right) \left( \pm \sqrt[d]{| \kappa_d(\epsilon_i) |} \mathbf{a}_i \right)^{\otimes d} & d \text{ even.} 
\end{cases}
\end{equation}
Tensor decomposition will therefore recover the columns of $A$ up to the factors $\pm \sqrt[d]{| \kappa_d(\epsilon_i) |}$.

Consider the LCD setting of~\eqref{eq:XandZ}. We have $X = FZ = F(I - \Lambda)^{-1} \epsilon$. The discussion above shows that the product $F(I - \Lambda)^{-1} \in \RR^{p \times q}$ can be recovered (up to permutation and scaling), since the entries of the random vector $\epsilon$ are independent. 
However, it is not possible to recover the latent DAG $\Gcal$ from the product $F(I - \Lambda)^{-1}$: a solution with empty DAG (that is, independent $Z$ variables) is always consistent with the observations, since
$$ F(I - \Lambda)^{-1} \epsilon = \widetilde{F} \widetilde{Z}, $$
where $\widetilde{F} = F(I - \Lambda)^{-1}$ and $\widetilde{Z} = \epsilon$. This demonstrates the need for observations of $X$ under multiple contexts.

\subsection{Coupling contexts}\label{subsec:align_output}

Distributions $X^{(k)}$ 
are linear mixtures of independent variables, since
$X^{(k)} = F(I - \Lambda^{(k)})^{-1} \epsilon^{(k)}$, 
where the entries of $\epsilon^{(k)}$ are independent. 
Our goal is to recover the parameters $F$ and $\Lambda^{(k)}$ for all $k \in K \cup \{0\}$.
We first prove the following.

\begin{proposition}
\label{prop:with_DP}
    Consider LCD under Assumption~\ref{assumption:main}. Then we can recover $q$ and the matrices $F(I - \Lambda^{(k)})^{-1}$ up to scaling and permutation for all $k \in K \cup \{ 0 \}$; i.e., we can recover
\begin{equation}
\label{eqn:recovered_matrix}
    F (I-\Lambda^{(k)})^{-1} D^{(k)} P^{(k)} \in \RR^{p \times q},
\end{equation}
where $D^{(k)} \in \RR^{q \times q}$ is diagonal, with non-zero diagonal entries, and $P^{(k)} \in \RR^{q \times q}$ is a permutation matrix. The diagonal matrix $D^{(k)}$ can be assumed to have entries
\begin{equation}
    \label{eqn:formula_for_D}
    D^{(k)}_{i,i} = 
 \begin{cases} 
 \phantom{\pm} \sqrt[d_i]{\kappa_{d_i}(\epsilon^{(k)}_i)} & d_i \text{ odd}, \\ 
\pm \sqrt[d_i]{|\kappa_{d_i}(\epsilon^{(k)}_i)|} & d_i \text{ even}, 
\end{cases}
\end{equation}
where, for all $i\in [q]$, $d_i$ is large enough ($d_i\geq 3(q-1)$ suffices) and satisfies $\kappa_{d_i}(\epsilon^{(k)}_i)\neq 0$.
\end{proposition}

\begin{proof}
We have $X^{(k)} = A^{(k)} \epsilon^{(k)}$, where $A^{(k)} = F(I - \Lambda^{(k)})^{-1}$ and $k$ ranges over contexts $K \cup \{ 0 \}$. 
We first prove the result under an additional assumption, that there exists a single number $d \geq 3$ that satisfies: 
\begin{enumerate} \item[(a)]
the tensor decomposition  
\begin{equation}
\label{eqn:cumulant_decomp2}
     \kappa_d(X^{(k)}) = \sum_{i = 1}^q \kappa_d(\epsilon_i^{(k)}) (\mathbf{a}^{(k)})_i^{\otimes d}
\end{equation}
is unique for all contexts $k$, where $(\mathbf{a}^{(k)})_i$ is the $i$-th column of $A^{(k)} = F (I-\Lambda^{(k)})^{-1}$,
    \item[(b)] $\kappa_d(\epsilon^{(k)}_i)\neq 0$ for all contexts $k$ and all $i \in [q]$,
    \item[(c)] $\kappa_d(\epsilon^{(i_k)}_{i_k}) \neq \pm1$ for all contexts $k$.
\end{enumerate}   
Fix such a $d$.
No pair of columns of $A^{(k)}$ are collinear, since collinearity is a Zariski closed condition with non-empty complement and the entries of $F$ and the non-zero values $\lambda^{(k)}_{i,j}$ are generic, by Assumption~\ref{assumption:main}(b).
Hence Proposition~\ref{prop:non_zero_coefficients} applies, and we recover the mixing matrix $A^{(k)}$ up to permutation and scaling; i.e., we recover the matrices in~\eqref{eqn:recovered_matrix}. The number of columns of these matrices is $q$.
Absorbing the coefficients of the tensor decomposition into the vectors as in~\eqref{eqn:lambdas_from_tensor}, the diagonal matrices in~\eqref{eqn:recovered_matrix} satisfy~\eqref{eqn:formula_for_D}
for every $i \in [q]$, $k \in K$,
where $d_i = d$ for all $i$.

We now show why such a $d$ as above is not required.
Part (a) holds for any $d \geq 3(q-1)$, see the proof of Proposition~\ref{prop:non_zero_coefficients}. 
Part (b) is subtle: the existence of a sufficiently large $d$ with $\kappa_d(\epsilon_i^{(k)}) \neq 0$ is equivalent to Assumption~\ref{assumption:main}(a) that the distribution $\epsilon_i^{(k)}$ is non-Gaussian, by Marcinkiewicz's theorem~\cite{marcinkiewicz1939propriete}. 
However, this does not imply the existence of a \emph{common} $d$ with that property, as we assumed above. If such a common $d$ does not exist, we instead recover the columns of $F (I-\Lambda^{(k)})^{-1}$ up to permutation and scaling, as well as the entries of $D^{(k)}$, using a \emph{set} of large enough cumulants $\kappa_{d_1}(X^{(k)}), \ldots, \kappa_{d_m}(X^{(k)})$ such that for all $i$ there exists $\ell \in [m]$ with $\kappa_{d_\ell}(\epsilon_i^{(k)}) \neq 0$. The non-Gaussianity assures that such a set exists, and the number of non-collinear vectors recovered from these tensor decompositions is $q$.
Part (c) can be avoided in the same way as (b), using Assumption~\ref{assumption:main}(c).
\end{proof}

Column scaling and permutation as in Proposition~\ref{prop:with_DP} have natural interpretations in LCD: there is no natural order on the latent variables, and they can be re-scaled without affecting membership in the model, see Remark~\ref{rmk:benign}. 
The goal of this section is to show that it is possible to fix an order and scaling of latent variables that is consistent across contexts. The upshot is the following result. 

\begin{proposition}
\label{prop:upshot}
Consider LCD under Assumption~\ref{assumption:main}. Then we can recover the number of latent nodes $q$ and the matrices
\begin{equation}\label{eq:Ak}
    A^{(k)} := F(I-\Lambda^{(k)})^{-1} \quad \text{for all} \quad k \in K \cup \{0\}.
\end{equation}
\end{proposition}

We fix a scaling of errors and an order on latent variables when $k=0$, as follows. 

\begin{proposition}
\label{prop:wlog_i}
Without loss of generality $P^{(0)} = D^{(0)} = I$.
\end{proposition}

\begin{proof}
We have recovered $A DP$ for some scaling $D$ and permutation $P$, by Proposition~\ref{prop:with_DP}, 
where we drop the superscripts since we refer only to the observational context.
The permutation $P$ orders the latent variables. 
We fix it to be the identity, thereby fixing an order of latent variables. 
We now consider $D$. Define
$\widetilde{F} = F D$ and $\widetilde{\Lambda} = D^{-1} \Lambda D$. 
Then 
$$ F(I - \Lambda)^{-1} D = \widetilde{F}(I - \widetilde{\Lambda})^{-1} .$$
and matrices $\Lambda$ and $\widetilde{\Lambda}$ have the same support. Hence $\widetilde{F}$ and $\widetilde{\Lambda}$ are valid parameters in the model, so we can without loss of generality set $D = I$. 
\end{proof}

The choice in Proposition \ref{prop:wlog_i} sets a non-zero cumulant $\kappa_{d_i}(\epsilon_i^{(0)})$ to $\pm 1$ for each $i \in [q]$, see~\eqref{eqn:formula_for_D}.
Hence $D^{(k)}_{i,i} 
\pm 1$
for all $i\neq i_k$, by Proposition \ref{prop:with_DP}, since $\epsilon^{(k)}$ and $\epsilon^{(0)}$ differ only at the intervention target $i_k$.
We now compare $A^{(0)}$ and $A^{(k)}D^{(k)}P^{(k)}$. The parents of a node $j$ are the set $\pa_\Gcal(j) = \{ i \in \Gcal \, | \, i \to j \in \Gcal\}$ and the ancestors of $j$ are $\an_\Gcal(j) = \{ i \in \Gcal \, | \, i \to \cdots \to j \in \Gcal\}$. 
We drop the subscript since $\Gcal$ is fixed.

\begin{proposition}\label{prop:3cases_target_permutation}
    Recall that $i_k \in [q]$ is the intervention target of context $k$ and let $j \in [q]$. 
Assume that $F$ is generic and that the non-zero entries of $\Lambda^{(k)}$ are generic.
    Then one of three possibilities arises.
    \begin{enumerate}
    \item[(i)] $j=i_k$ and the $j$-th column of $A^{(0)}$ equals one of the columns of $A^{(k)}D^{(k)}P^{(k)}$ up to a scaling that is not $\pm 1$;
    \item[(ii)] $j \not\in \an(i_k)\cup \{i_k\}$ and the $j$-th column of $A^{(0)}$ equals one of the columns of $A^{(k)}D^{(k)}P^{(k)}$, up to sign;
    \item[(iii)] $j\in \an(i_k)$ and the $j$-th column of $A^{(0)}$ is not parallel to any of the columns of $A^{(k)}D^{(k)}P^{(k)}$. 
\end{enumerate}
\end{proposition}
\begin{proof}
(i) Assume $j = i_k$, and let $j_k = P^{(k)}(j)$. The $(i,j_k)$ entry of $A^{(k)}D^{(k)}P^{(k)}$ is 
\begin{equation}\label{eq:compare_columns_contexts}
    (A^{(k)}D^{(k)})_{i,j} = A^{(k)}_{i,j} D^{(k)}_{j,j} = A^{(0)}_{i,j} D^{(k)}_{j,j}.
\end{equation}
Therefore, the $j_k$-th column of $A^{(k)}D^{(k)}P^{(k)}$ is a non-trivial (not $0$ or $\pm 1$) multiple of the $j$-th column of $A^{(0)}$, since $D^{(k)}_{j,j}\neq \pm 1$ when $j$ is the intervention target.

(ii) Assume $j \not\in \an(i_k)\cup \{i_k\}$ and let $j_k = P^{(k)}(j)$. The  chain of equalities in \eqref{eq:compare_columns_contexts} holds true, but $D^{(k)}_{j,j}=\pm 1$. Hence, the $j_k$-th column of $A^{(k)} D^{(k)}P^{(k)}$ is the $j$-th column of $A^{(0)}$ up to sign.

(iii) Let $j\in \an(i_k)$. Assume for contradiction that there exists a column $r$ of $A^{(k)}D^{(k)}P^{(k)}$ that is parallel to the $j$-th column of $A^{(0)}$. Let $r_k = P^{(k)}(r)$. Then, there exists $\alpha$ such that for every $i\in [p]$, 
\[
\sum_{\ell \in [q]} f_{i,\ell} (I-\Lambda^{(0)})^{-1}_{\ell, j} = \alpha \sum_{\ell \in [q]} f_{i,\ell} (I-\Lambda^{(k)})^{-1}_{\ell, r_k} D^{(k)}_{r_k, r_k}.
\]
By genericity of $F$ and $\Lambda^{(k)}$, the equality holds if and only if it holds for the coefficient of every $f_{i,\ell}$ independently. It is therefore equivalent to 
\[
(I-\Lambda^{(0)})^{-1}_{\ell, j} = \alpha (I-\Lambda^{(k)})^{-1}_{\ell, r_k} D^{(k)}_{r_k, r_k}
\]
for all $\ell\in [q]$. If $\ell = j$, then by genericity of $\Lambda^{(k)}$ we have $r_k = j$ and $\alpha = D^{(k)}_{j,j}$. However, since $D^{(k)}_{j, j} = \pm 1$, this leads to the equality
\[
(I-\Lambda^{(0)})^{-1}_{\ell, j} = (I-\Lambda^{(k)})^{-1}_{\ell, j}
\]
for every $\ell \in [q]$, which implies by genericity that $\lambda^{(0)}_{i_k,m} = \lambda^{(k)}_{i_k,m}$ for every $m$, a contradiction.
\end{proof}

Proposition \ref{prop:rank1} recovers the target of each intervention. It also recovers the ancestors of each latent node. That is, it recovers the transitive closure $\overline{\Gcal}$, providing a simpler proof of the following result, proven without the non-Gaussian assumption in \cite[Theorem 1]{SSBU23:LinearCausalDisentanglementInterventions}.

\begin{corollary}
\label{cor:transitive_closure}
Consider LCD under Assumption~\ref{assumption:main}
    with one intervention (either perfect or soft) on each latent node. Then we can recover the transitive closure $\overline{\Gcal}$ of the latent DAG $\Gcal$.
\end{corollary}

\begin{remark}[Paths in $\Gcal$]
\label{rmk:paths}
Entry $(i,j)$ of the matrix $(I-\Lambda^{(k)})^{-1}$ is a sum over all the paths $j \to \cdots \to i$ in $\Gcal$, where each path contributes the product $\lambda^{(k)}_{m,n}$ over all edges $n\to m$ in the path. For instance, for the DAG $3\to 2\to 1$ we have $(I-\Lambda^{(k)})^{-1}_{1,3} = \lambda^{(k)}_{1,2} \lambda^{(k)}_{2,3}$. Adding the edge $3\to 1$ gives $(I-\Lambda^{(k)})^{-1}_{1,3} = \lambda^{(k)}_{1,2} \lambda^{(k)}_{2,3} + \lambda^{(k)}_{1,3}$.
\end{remark}

Proposition~\ref{prop:3cases_target_permutation} partially recovers the permutation $P^{(k)}$, as it pairs all columns $j \notin \an(i_k)$. We can therefore assume without loss of generality that $i_k = k$ and that $P^{(k)}_{i,j} = \delta_{i,j}$ for every $j \not\in\an(k)$.
We are left to pair the columns of $\an(k)$.

\begin{proposition}\label{prop:rank1}
    For $j_1, j_2 \in \an(k)$, there exists $\alpha \in \RR$ such that 
    \[
    (I-\Lambda^{(0)})^{-1}_{i,j_1} - \left( (I-\Lambda^{(k)})^{-1} D^{(k)} \right)_{i,j_2} = \alpha \left( (I-\Lambda^{(0)})^{-1}_{i,k} - \left( (I-\Lambda^{(k)})^{-1} D^{(k)} \right)_{i,k} \right)
    \]
    for all $i \in [q]$, if and only if $j_1 = j_2$ and $D^{(k)}_{j_1,j_1} = 1$.
\end{proposition}
\begin{proof}
    Fix $j:= j_1 = j_2$ and assume $D^{(k)}_{j_1,j_1} = 1$.
    The left hand side is a sum over all paths from $Z_{j}$ to $Z_i$, through $Z_k$, since the paths that do not go through $Z_k$ cancel:
    \begin{align*}
        (I-\Lambda^{(0)})^{-1}_{i, j} - (I-\Lambda^{(k)})^{-1}_{i, j} &= (I-\Lambda^{(0)})^{-1}_{i,k} (I-\Lambda^{(0)})^{-1}_{k, j} - (I-\Lambda^{(k)})^{-1}_{i,k} (I-\Lambda^{(k)})^{-1}_{k, j} \\
        &= (I-\Lambda^{(0)})^{-1}_{i,k} \left( (I-\Lambda^{(0)})^{-1}_{k, j} - (I-\Lambda^{(k)})^{-1}_{k, j} \right) \\
        &= \tfrac{ (I-\Lambda^{(0)})^{-1}_{k, j} - (I-\Lambda^{(k)})^{-1}_{k, j}}{1-D^{(k)}_{k,k}} \left( (I-\Lambda^{(0)})^{-1}_{i,k} - (I-\Lambda^{(k)})^{-1}_{i,k}D^{(k)}_{k,k} \right) ,
    \end{align*}
    where we used $(I-\Lambda^{(0)})^{-1}_{i,k} - (I-\Lambda^{(k)})^{-1}_{i,k} D^{(k)}_{k,k} = (I-\Lambda^{(0)})^{-1}_{i,k} (1-D^{(k)}_{k,k})$. This proves one direction.

    Assume conversely that the equality in the statement holds for some $j_1,j_2 \in \an(k)$, and let  $i = j_1$. Then
    \[
    (I-\Lambda^{(0)})^{-1}_{j_1,j_1} - \left( (I-\Lambda^{(k)})^{-1} D^{(k)} \right)_{j_1,j_2} = \alpha \left( (I-\Lambda^{(0)})^{-1}_{j_1,k} - \left( (I-\Lambda^{(k)})^{-1} D^{(k)} \right)_{j_1,k} \right).
    \]
    The right-hand side is zero since the latent graph is a DAG and $j_1 \in \an(k)$. Hence
    \[
    0 = (I-\Lambda^{(0)})^{-1}_{j_1,j_1} - \left( (I-\Lambda^{(k)})^{-1} D^{(k)} \right)_{j_1,j_2} = 1 \pm (I-\Lambda^{(0)})^{-1}_{j_1,j_2}
    \]
    where we used that there are no paths from $j_2$ to $k$ and $D^{(k)}_{j_2,j_2} = \pm 1$. Therefore, $ (I-\Lambda^{(0)})^{-1}_{j_1,j_2} = 1$, which implies by genericity that $j_1 = j_2$ and $D^{(k)}_{j_2,j_2} = 1$.
\end{proof}
    
\begin{corollary}\label{coroll:rank1}
    For every $k$ and for generic parameters in $F, \Lambda^{(0)}, \Lambda^{(k)}$,
    we have
    \[
    \rank \left( A^{(0)} - A^{(k)} D^{(k)} P^{(k)} \right) = 1
    \]
    if and only if $P^{(k)} = I$ and $D^{(k)}_{j,j} = 1$ for all $j\neq k$.
\end{corollary}
\begin{proof}
    A matrix has rank one if and only if all its columns are scalar multiples. Therefore, our claim is equivalent to the existence for every $j\in [q]$ of some $\alpha \in \RR$ such that 
    \begin{multline}\label{eq:rank1proof}
    \sum_{i\in [q]} f_{\ell,i} \left((I-\Lambda^{(0)})^{-1}_{i,j} - \left( (I-\Lambda^{(k)})^{-1} D^{(k)}P^{(k)} \right)_{i,j} \right) \\ 
    =\alpha \sum_{i\in [q]} f_{\ell,i} \left( (I-\Lambda^{(0)})^{-1}_{i,k} - \left( (I-\Lambda^{(k)})^{-1} D^{(k)}P^{(k)} \right)_{i,k} \right),
    \end{multline}
    for every $\ell \in [p]$. If we treat the parameters in $F, \Lambda^{(0)}, \Lambda^{(k)}, D^{(k)}_{k,k}$ as indeterminates, the equation holds if and only if all the summands are equal.
    Analogously, this is the case if the $f_{\ell,i}$ parameters are generic.

    For $j\not \in \an(k)$, we have $P^{(k)}_{i,j} = \delta_{i,j}$. Assume for contradiction that $D^{(k)}_{j,j} = -1$. Then, for $i=k$, the left-hand side of \eqref{eq:rank1proof} is $0$ and the right-hand side is $\alpha (1-D^{(k)}_{k,k})$, which forces $\alpha=0$. However, for $i=j$, the left-hand side is $2$, so $\alpha\neq 0$, a contradiction. This forces $D^{(k)}_{j,j} = 1$ for the non-ancestors of $k$. Putting this together with Proposition \ref{prop:rank1}, we deduce that the matrix $(I-\Lambda^{(0)})^{-1} - (I-\Lambda^{(k)})^{-1} D^{(k)} P^{(k)}$ has rank at most $1$ if and only if $P^{(k)} = I$ and $D^{(k)}_{j,j} = 1$ for all $j\neq k$. Moreover, because $D^{(k)}_{k,k}\neq 1$, the $k$-th column of the difference matrix is non-zero, hence the rank is exactly $1$.
\end{proof}

\begin{proof}[Proof of Proposition~\ref{prop:upshot}]
    We recover the matrices $A^{(k)}$ up to scaling and permutation, by Proposition \ref{prop:with_DP}. The upshot of Corollary~\ref{coroll:rank1} is that we can identify the target $i_k$ of the intervention and the permutation. Hence
    we can get rid of $P^{(k)}$ by right multiplication with its transpose. Now the $i_k$-th column of $F(I-\Lambda^{(0)})^{-1}$ differs from the $i_k$-th column of $F(I-\Lambda^{(k)})^{-1} D^{(k)}$ by the scaling $D^{(k)}_{i_k,i_k}$, so we can also recover the diagonal matrix, and hence $A^{(k)}$ itself.
\end{proof}

\begin{remark}\label{rmk:q<p_sec2}
    While Proposition~\ref{prop:upshot} holds for any $p, q \geq 2$, the proof is simpler when $q\leq p$. Then, the Moore-Penrose pseudo-inverse satisfies
    \[
    \left( F (I-\Lambda^{(k)})^{-1} D^{(k)} P^{(k)} \right)^+ = (P^{(k)})^\top (D^{(k)})^{-1} (I-\Lambda^{(k)}) F^+.
    \]
    Finding the permutation and intervention targets is done as follows. There is just one row in the pseudo-inverse of the context $k$ that does not appear in the pseudo-inverse of the observational context. Hence it indexes the intervention target. The permutation is found by matching the remaining rows of the two matrices.
    The expression relating the psuedo-inverse of the product to the product of psuedo-inverses does not hold in general when $q > p$.
\end{remark}

\section{Recovery via interventions}\label{sec:recovery_interventions}

In this section, we identify when two latent graphs and parameters $F, \Lambda^{(k)}$ give the same distributions $X^{(k)}$.
At this stage, we have access to the matrices $A^{(k)}$ in \eqref{eq:Ak}, 
by Proposition~\ref{prop:upshot}.

\begin{proposition}
\label{prop:matrices_vs_distributions}
    Distributions $F(I-\Lambda^{(k)})^{-1} \epsilon^{(k)}$ and $\widetilde{F}(I-\widetilde{\Lambda}^{(k)})^{-1} \widetilde{\epsilon}^{(k)}$ coincide for all $k \in K \cup \{ 0 \}$ if and only if there exists a reordering of the sets $\{\epsilon^{(0)}_i\}$, $\{\widetilde{\epsilon}^{(0)}_i\}$ and a rescaling of $F, \Lambda^{(0)}, \widetilde{F}, \widetilde{\Lambda}^{(0)}$ via~\eqref{eqn:equiv} such that $F(I-\Lambda^{(k)})^{-1} = \widetilde{F}(I-\widetilde{\Lambda}^{(k)})^{-1}$ for all $k \in K \cup \{ 0 \}$.
\end{proposition}

\begin{proof}
Define $A^{(k)} =F(I-\Lambda^{(k)})^{-1}$ and $ \widetilde{A}^{(k)} = \widetilde{F}(I-\widetilde{\Lambda}^{(k)})^{-1}$.
    The equality of matrices $A^{(k)}$ and $\widetilde{A}^{(k)}$ implies the equality of the distributions $X^{(k)} = A^{(k)} \epsilon^{(k)}$ and $\widetilde{X}^{(k)} = \widetilde{A}^{(k)} \epsilon^{(k)}$.
    Conversely, assume that distributions $X^{(k)}$ and $\widetilde{X}^{(k)}$ coincide. Then, we have the equality of cumulants $\kappa_d(X^{(k)}) = \kappa_d(\widetilde{X}^{(k)})$ for all $k$ and $d$. To simplify the exposition, we assume that there exists $d$ as in the proof of Proposition \ref{prop:with_DP} (this assumption can be avoided using the same argument as in the proof of Proposition \ref{prop:with_DP}).
    For this fixed $d$ and for each context $k$, the tensor decomposition of $\kappa_d(X^{(k)})$ is unique up to rescaling and permutation. Since the cumulant is the same for both distributions, from the decomposition we get 
    \[
    F(I-\Lambda^{(k)})^{-1} D^{(k)} P^{(k)} = \widetilde{F}(I-\widetilde{\Lambda}^{(k)})^{-1} \widetilde{D}^{(k)} \widetilde{P}^{(k)}.
    \]
    Fix $k=0$. We can reorder the variables $\epsilon_i^{(0)}$ to set $P^{(0)} = I$ and we can absorb $D^{(0)}$ into $F$ and $\Lambda^{(0)}$, as in Proposition \ref{prop:wlog_i}. Analogously, we can do the same in the tilde setting. Therefore, up to reordering and rescaling via~\eqref{eqn:equiv} we have 
    \[
    F(I-\Lambda^{(0)})^{-1} = \widetilde{F}(I-\widetilde{\Lambda}^{(0)})^{-1}.
    \]
    Then, Corollary \ref{coroll:rank1} implies that there exists a unique choice of signs of the diagonal matrices and a unique permutation matrix $Q$ satisfying
    \[
    \rank \left( A^{(0)} - A^{(k)} D^{(k)} P^{(k)} Q \right) = 1 = \rank \left( \widetilde{A}^{(0)} - \widetilde{A}^{(k)} \widetilde{D}^{(k)} \widetilde{P}^{(k)} Q \right),
    \]
    and $Q = (P^{(k)})^\top$ and $Q = (\widetilde{P}^{(k)})^\top$. Therefore, $P^{(k)}=\widetilde{P}^{(k)}$, and by comparing the intervened columns of the difference matrices we have $D^{(k)}=\widetilde{D}^{(k)}$. This implies $A^{(k)}=\widetilde{A}^{(k)}$.
\end{proof}

The upshot is that solving Problem \ref{prob:main} is equivalent to solve the following problem. 

\begin{problem}
\label{prob:polynomials}
    Given a generic matrix $\widetilde{F} \in \R^{p\times q}$ and matrices $\widetilde{\Lambda}^{(0)}, \ldots, \widetilde{\Lambda}^{(q)} \in \R^{q\times q}$ constructed according to a model with DAG $\widetilde{\Gcal}$, with generic non-zero entries, do there exist a generic matrix $F \in \R^{p\times q}$, and matrices $\Lambda^{(0)}, \ldots, \Lambda^{(q)} \in \R^{q\times q}$ constructed according to a model with DAG $\Gcal$, such that \begin{equation}\label{eq:system_equations_identifiability}
        F(I-\Lambda^{(k)})^{-1} = \widetilde{F}(I-\widetilde{\Lambda}^{(k)})^{-1}
    \end{equation}
    for all $ k \in K \cup \{0\}$? If so, how are the DAGs and the corresponding matrices related?
\end{problem}

We solve the system of polynomial equations~\eqref{eq:system_equations_identifiability}. The solution is unique if and only if the DAG and the matrices are identifiable. 
Otherwise, the set of solutions is the set of possible DAGs and space of possible parameters.

From now on, unless otherwise stated, we assume that we have the observational context and one intervention per latent node. We re-index contexts so that the $k$-th intervention (either soft or perfect) is on $Z_k$, hence $K = [q]$.

\subsection{A linear system}

Let $A^{(k)} = \widetilde{F}(I-\widetilde{\Lambda}^{(k)})^{-1}$ and let $\mathcal{S}$ be the space of solutions to \eqref{eq:system_equations_identifiability}. The algebraic variety $\mathcal{S}$ is associated to the ideal
\begin{equation*}
    \mathcal{I} = \langle F(I-\Lambda^{(k)})^{-1} - A^{(k)}, \; k = 0,\ldots,q \rangle.
\end{equation*}
The matrices $F, \Lambda^{(k)}$ are filled with indeterminates.
Each point of $\mathcal{S}$ provides a graph and parameters compatible with the given model.
At first sight, $\mathcal{S}$ might have high degree, since the degree of the generators can reach $q+1$. However, there is a simpler set of generators:
\begin{equation}\label{eq:def_ideal}
    \mathcal{I} = \langle F - A^{(k)}(I-\Lambda^{(k)}), \; k = 0,\ldots,q \rangle.
\end{equation}
Assuming the $A^{(k)}$ are known, $\mathcal{I}$ has $(q+1)qp$ linear generators in a polynomial ring with $qp+(q+1)|e(\Gcal)|$ indeterminates. We find a set of minimal generators for $\mathcal{I}$ to compute the dimension of the associated algebraic variety; i.e., to find the identifiability of the parameters.

\begin{proposition}
\label{prop:nonlinear_covariance}
Consider the setup in Assumption~\ref{assumption:main} with $q$ soft interventions. When $d>2$, the space of parameters $F$ and $\Lambda^{(k)}$, $k \in [q] \cup \{ 0 \}$, such that $\kappa_d(X^{(k)})$ is a given tensor is a linear space. When $d=2$, for any $q \in \N$ there exists a DAG on $q$ nodes such that the space of parameters for which $\kappa_2(X^{(k)})$ is a given matrix for all $k \in [q] \cup \{ 0 \}$ is non-linear.
\end{proposition}

\begin{proof}
    The case $d>2$ follows from \eqref{eq:def_ideal}. For the case $d = 2$,
    consider a model on two latent nodes with one edge $2\to 1$, with parameters $\widetilde{F} = \left(\begin{smallmatrix}
        2 & 3 \\
        5 & 11
    \end{smallmatrix}\right)$, $\widetilde{\lambda}^{(0)}_{1,2} = \widetilde{\lambda}^{(2)}_{1,2} = 7$, $\widetilde{\lambda}^{(1)}_{1,2} = 13$. 
    Symbolic computation with, e.g., \texttt{Macaulay2} or \texttt{Oscar.jl} shows that the space of parameters that satisfy $\kappa_2(X^{(k)}) = \kappa_2(\widetilde{X}^{(k)})$ for $k \in \{0,1,2 \}$ is $1$-dimensional and of degree $8$. It is the union of $6$ irreducible components, four linear and two quadratic. The same happens for generic parameters.
    We can embed this DAG into a DAG on $q$ nodes with only one edge $2\to 1$. Then, the space of solutions has the same dimension ($=1$) and degree ($=8$) as the space of solutions for the DAG on two nodes.
\end{proof}
Proposition \ref{prop:nonlinear_covariance} shows that the non-Gaussianity assumption is required in Theorem~\ref{thm:soft1}.
We conclude this subsection with an example, to see the linear structure of $\mathcal{I}$.

\begin{example}\label{ex:graph4nodes_lin_ideal}
    Consider the latent DAG
    \begin{center}
    \begin{tikzpicture}
    \node at (-1.2,0) {$\Gcal \; = $};
    \begin{scope}[every node/.style={circle,thick,draw}]
        \node (4) at (0,0) {$4$};
        \node (3) at (2,0) {$3$};
        \node (2) at (4,0) {$2$};
        \node (1) at (6,0) {$1$};
    \end{scope}
    \begin{scope}[>={Stealth[PineGreen]},
                  every node/.style={fill=white,circle},
                  every edge/.style={draw=PineGreen,very thick}]
=        \path [->] (4) edge (3);
        \path [->] (2) edge (1);
        \path [->] (4) edge[bend right = -30] (2);
        \path [->] (3) edge[bend right = 30] (1);
    \end{scope}
    \end{tikzpicture}
    \end{center}    
with parameters
    \[
    F = \begin{pmatrix}
        2 & 6 & 10 & 1 \\
        2 & 9 & -3 & 8 \\
        -8 & 4 & 7 & 2 \\
        -9 & 8 & 2 & -5
    \end{pmatrix},
    \qquad
    \Lambda^{(0)} = \Lambda^{(4)} = \begin{pmatrix}
        0 & 9 & 3 & 0 \\
        0 & 0 & 0 & 10 \\
        0 & 0 & 0 & 7 \\
        0 & 0 & 0 & 0
    \end{pmatrix},
    \]
    \[
    \Lambda^{(1)} = \begin{pmatrix}
        0 & -5 & 8 & 0 \\
        0 & 0 & 0 & 10 \\
        0 & 0 & 0 & 7 \\
        0 & 0 & 0 & 0
    \end{pmatrix},
    \qquad
    \Lambda^{(2)} = \begin{pmatrix}
        0 & 9 & 3 & 0 \\
        0 & 0 & 0 & 2 \\
        0 & 0 & 0 & 7 \\
        0 & 0 & 0 & 0
    \end{pmatrix},
    \qquad
    \Lambda^{(3)} = \begin{pmatrix}
        0 & 9 & 3 & 0 \\
        0 & 0 & 0 & 10 \\
        0 & 0 & 0 & -1 \\
        0 & 0 & 0 & 0
    \end{pmatrix}.
    \]
    Then, assuming known $\Gcal$, the ideal $\mathcal{I}$ \eqref{eq:def_ideal} is minimally generated by $21$ linear polynomials
    \begin{gather*}
        f_{1,1} -2,\qquad f_{2,1} -2,\qquad f_{3,1} + 8,\qquad f_{4,1}+9, \\
          f_{1, 2} + 2 \lambda^{(0)}_{1, 2} - 24,\quad f_{2, 2} + 2 \lambda^{(0)}_{1, 2} - 27,\quad f_{3, 2} - 8 \lambda^{(0)}_{1, 2} + 68,\quad  f_{4, 2} - 9 \lambda^{(0)}_{1, 2} + 73, \\
           f_{1, 3} + 2 \lambda^{(0)}_{1, 3} - 16,\qquad f_{2, 3} + 2 \lambda^{(0)}_{1, 3} - 3,\qquad f_{3, 3} - 8 \lambda^{(0)}_{1, 3} + 17,\qquad f_{4, 3} - 9 \lambda^{(0)}_{1, 3} + 25, \\
           f_{1, 4} + \tfrac{172}{7} \lambda^{(0)}_{3, 4} - 173, \quad f_{2, 4} + \tfrac{177}{14} \lambda^{(0)}_{3, 4} - \tfrac{193}{2}, \quad f_{3, 4} - \tfrac{289}{7} \lambda^{(0)}_{3, 4} + 287,\quad f_{4, 4} - \tfrac{715}{14} \lambda^{(0)}_{3, 4} + \tfrac{725}{2},\\
           \lambda^{(1)}_{1,2} - \lambda^{(0)}_{1,2} +14,\quad \lambda^{(1)}_{1,3} - \lambda^{(0)}_{1,3} -5,\quad \lambda^{(2)}_{2,4} - \lambda^{(0)}_{2,4} +8,\quad \lambda^{(3)}_{3,4} - \lambda^{(0)}_{3,4} +8, \quad -14 \lambda^{(0)}_{2, 4} +5\lambda^{(0)}_{3, 4} +105.
    \end{gather*}
    These can be found by computing the primary decomposition of $\mathcal{I}$ in computer algebra software, such as \texttt{Macaulay2} \cite{M2} or \texttt{Oscar.jl} \cite{OSCAR,OSCAR-book}.
\end{example}

\subsection{Perfect interventions}\label{subsec:perfect_int}

When the interventions are perfect, namely $\lambda^{(k)}_{k,j} = 0$ for every $k\in [q]$, there is a unique solution to the linear system in \eqref{eq:def_ideal}. In other words, the ideal $\mathcal{I}$ is zero dimensional and defines a point. This is Theorem \ref{thm_intro_perfect}.

\begin{proof}[Proof of Theorem \ref{thm_intro_perfect}]
    Worst case necessity of one intervention per node for identifiability is a direct consequence of \cite[Proposition 5]{SSBU23:LinearCausalDisentanglementInterventions}.
    We prove sufficiency. We have matrices $A^{(k)} = F(I-\Lambda^{(k)})^{-1}$, by Proposition~\ref{prop:upshot}. 
    Pick $k,j \in [q]$ with $k\neq j$. Then,
    \begin{align*}
    \left( A^{(0)} - A^{(k)}\right)_{1,j} &= \sum_{\ell \in [q]} f_{1, \ell} \left( (I-\Lambda^{(0)})^{-1} - (I-\Lambda^{(k)})^{-1} \right)_{\ell, i} \\
    &= \sum_{\ell \in \de(k)} f_{1, \ell} (I-\Lambda^{(0)})^{-1}_{\ell, k}  \left( (I-\Lambda^{(0)})^{-1} - (I-\Lambda^{(k)})^{-1} \right)_{k, j} \\
    &= A^{(0)}_{1,k} \: (I-\Lambda^{(0)})^{-1}_{k, j}.
\end{align*}
    With this, we construct $(I-\Lambda^{(0)})^{-1}$ and hence recover $\Lambda^{(0)}$. We multiply $A^{(0)} (I-\Lambda^{(0)})$ to obtain $F$.
\end{proof}
The above result shows that $q$ perfect interventions are sufficient to recover the DAG and the parameters of a model. To find the parameters (and hence the latent DAG), one can solve the linear system~\eqref{eq:def_ideal} or follow the procedure in the proof.

\begin{remark}\label{rmk:q<p_sec3_perfect}
    When $q\leq p$, an alternative proof via pseudo-inverses exists, see Section~\ref{subsec:injective}.
\end{remark}

When $q>p$, the non-Gaussianity assumption is necessary for Theorem~\ref{thm_intro_perfect}, as follows.

\begin{proposition}\label{prop:cov_not_enough}
Consider LCD under Assumption~\ref{assumption:main} with perfect interventions and $q>p$. Then one perfect intervention on each latent node is not sufficient to recover the latent DAG $\Gcal$ and the parameters $F$ and $\Lambda^{(k)}$ from the covariance matrices of $X^{(k)}$.
\end{proposition}
\begin{proof}
    For $(p,q)=(2,3)$ and $\Gcal = \emptyset$, we compute the parameters $F$ and $\Lambda^{(k)}$ for which the covariance matrices $F(I-\Lambda^{(k)})^{-1}(D^{(k)})^2(I-\Lambda^{(k)})^{-\top} F^\top$ coincide with the true covariance matrices for $k = 0,1,2$. We choose and fix an ordering of the nodes, and we fix the scaling by imposing $D^{(0)}=I$. This space has dimension $2$, so the parameters cannot be recovered uniquely. We can embed this DAG in a DAG with $q$ nodes, for any $q$. Hence the $p\times p$ covariance matrices do not contain enough information to recover the parameters, when $p<q$.
\end{proof}

\subsection{Soft interventions}\label{subsec:soft_int}

In this section we compute the dimension of solutions of the linear system $F - A^{(k)}(I-\Lambda^{(k)}) = 0$
for $k = 0, \ldots , q$, under soft interventions.
For every $k$ and for every $\ell \in [p], j \in [q]$, we have 
\begin{equation*}
    f_{\ell, j} + \sum_{i \in \ch(j)} A^{(k)}_{\ell, i} \lambda^{(k)}_{i,j} = A^{(k)}_{\ell, j}.
\end{equation*}
There are $pq(q+1)$ equations in $pq+2|e(\Gcal)|$ indeterminates, namely $f_{\ell,j}$ for all $\ell \in [p]$, $j\in[q]$, and $\lambda^{(0)}_{i,j}$ for all $(j\to i) \in e(\Gcal)$, and $\lambda^{(k)}_{k,j}$ for all $(j\to k) \in e(\Gcal)$.
For each $(\ell,j)$, we subtract the equation for $k=0$ from the equations for $k\in [q]$. Then the $\big( pq(q+1) \big) \times \big( pq+2|e(\Gcal)| \big)$ matrix of the linear system has block structure 
\[
\left(
\begin{array}{c|c}
    I_{pq} & \star \\
    \hline
    0 & \star 
\end{array}
\right).
\]
We can focus on the $\big( pq^2 \big) \times \big( 2|e(\Gcal)| \big)$ bottom-right block, involving only the indeterminates $\Lambda^{(k)}$. The equations of this smaller linear system are $A^{(k)}(I-\Lambda^{(k)}) - A^{(0)}(I-\Lambda^{(0)}) = 0$, or
\begin{equation}\label{eq:lambdak_lin_sys}
    \sum_{i \in \ch(j)} A^{(k)}_{\ell, i} \lambda^{(k)}_{i,j} - \sum_{i \in \ch(j)} A^{(0)}_{\ell, i} \lambda^{(0)}_{i,j} = A^{(k)}_{\ell, j} - A^{(0)}_{\ell, j},
\end{equation}
for $k, j \in [q]$, $\ell \in [p]$. There are three cases:
\begin{enumerate}
    \item If $k \not \in \ch(j)$, then \eqref{eq:lambdak_lin_sys} becomes
\begin{equation}\label{eq:linear_sys_no_ch}
\sum_{i \in \ch(j)}(A^{(k)}_{\ell, i} - A^{(0)}_{\ell, i})\lambda^{(0)}_{i,j} = (A^{(k)}_{\ell, j} - A^{(0)}_{\ell, j}).
\end{equation}
The $(\ell,i)$ entry of $A^{(k)}-A^{(0)}$ is by definition $\sum_{n\in \de(i)} f_{\ell,n} \left( (I-\Lambda^{(k)})^{-1} - (I-\Lambda^{(0)})^{-1}\right)_{n,i}$.
If $k \not \in \de(j)$, then $\left( (I-\Lambda^{(k)})^{-1} - (I-\Lambda^{(0)})^{-1}\right)_{n,i} = 0$ for every $n$ since by construction $\de(j)\supset\de(i)$. Hence,~\eqref{eq:linear_sys_no_ch} reads $0=0$ and it imposes no condition on our indeterminates.

\item If $k\in \de(j)\setminus\ch(j)$, then 
$\left( (I-\Lambda^{(k)})^{-1} - (I-\Lambda^{(0)})^{-1}\right)_{n,i}\neq 0$, since there is a path from $i$ to $n$ through $k$. Such a path must exist for some $n$, since $k$ is a descendant of some $i$. Hence the coefficients of \eqref{eq:linear_sys_no_ch} are non-zero, and we get linear conditions on the indeterminates.

\item Finally, if $k\in \ch(j)$, we get an expression for $\lambda^{(k)}_{k,j}$ in terms of the $\lambda^{(0)}_{i,j}$:
\begin{equation}\label{eq:linear_sys_lambdak}
\begin{aligned}
    \lambda^{(k)}_{k,j} &= \frac{1}{A^{(k)}_{\ell, k} } \left( A^{(0)}_{\ell, k}\lambda^{(0)}_{k,j} + \sum_{\substack{i \in \ch(j)\\ i\neq k}}(A^{(0)}_{\ell, i} - A^{(k)}_{\ell, i})\lambda^{(0)}_{i,j} + (A^{(k)}_{\ell, j} - A^{(0)}_{\ell, j}) \right) \\
    &= \lambda^{(0)}_{k,j} + \sum_{\substack{i \in \ch(j)\\ i\neq k}}\frac{A^{(0)}_{\ell, i} - A^{(k)}_{\ell, i}}{A^{(0)}_{\ell, k} }\lambda^{(0)}_{i,j} + \frac{A^{(k)}_{\ell, j} - A^{(0)}_{\ell, j}}{A^{(0)}_{\ell, k} },
\end{aligned}   
\end{equation}
where we used $A^{(k)}_{\ell,k} - A^{(0)}_{\ell,k} = 0$ because $\left( (I-\Lambda^{(k)})^{-1} - (I-\Lambda^{(0)})^{-1}\right)_{n,k} = 0$ for every $n$. We get~\eqref{eq:linear_sys_lambdak} for every $\ell \in [p]$. However, most equations are redundant. 
\end{enumerate}

The following result mimics Proposition \ref{prop:rank1}.
\begin{proposition}\label{prop:difference_lambdas_rank1}
    For $k \in [q]$, let $\Delta^{(k)} = (I-\Lambda^{(k)})^{-1} - (I-\Lambda^{(0)})^{-1}$. Then,
$   \rank \left( \Delta^{(k)} \right) \leq 1$,
with equality if and only if $\an (k) \neq \emptyset$.
\end{proposition}
\begin{proof}
    Fix $k\in [q]$ and recall that the $(i,j)$ entry of $(I-\Lambda^{(k)})^{-1}$ is the sum of all paths from $Z_j$ to $Z_i$, where a path is encoded as the product of $\lambda^{(k)}_{m,n}$ for all edges $n\to m$ in the path. Then, the only non-zero columns of $\Delta^{(k)}$ are those indexed by $j$ for $j\in \an(k)$. We prove that these columns are multiple of each other. Let $j_1, j_2 \in \an(k)$, then 
    \begin{align*}
        \Delta^{(k)}_{i, j_m} &=  (I-\Lambda^{(k)})^{-1}_{i,k} (I-\Lambda^{(k)})^{-1}_{k, j_m} - (I-\Lambda^{(0)})^{-1}_{i,k} (I-\Lambda^{(0)})^{-1}_{k, j_m} \\
        &= (I-\Lambda^{(0)})^{-1}_{i,k} \Delta^{(k)}_{k, j_m}
    \end{align*}
    for $m=1,2$. Hence, for every $i \in [q]$, the $(i,j_1)$ entry equals the $(i,j_2)$ entry up to $\frac{\Delta^{(k)}_{k, j_2}}{\Delta^{(k)}_{k, j_1}}$.
\end{proof}

For generic parameters we have $\rank (A^{(k)} - A^{(0)}) \leq 1$ for all $k \in [q]$, with equality whenever $\an(k)\neq \emptyset$,
by Proposition \ref{prop:difference_lambdas_rank1}, with proof is analogous to that of Corollary \ref{coroll:rank1}.
Hence the conditions in \eqref{eq:linear_sys_no_ch} are equivalent for all $\ell \in [p]$, and the same is true of the conditions in \eqref{eq:linear_sys_lambdak}. This reduces the size of the linear system, taking only the equations for $\ell = 1 \in [p]$. We obtain a reduced matrix of the linear system 
\begin{equation}\label{eq:matrix_lin_sys_q>p}
\left(
\begin{array}{c|c|c}
    I_{pq} & 0 & \star \\
    \hline
    0 & I_{|e(\Gcal)|} & \star \\
    \hline
    0 & 0 & \star
\end{array}
\right),
\end{equation}
where the top block writes $F$ in terms of $\Lambda^{(0)}$, the second block writes $\Lambda^{(k)}$ in terms of $\Lambda^{(0)}$, and the bottom block gives the conditions~\eqref{eq:linear_sys_no_ch} on $\Lambda^{(0)}$. The latter are $\sum_{j\in [q]} |\de(j)\setminus\ch(j)|$ equations in $\sum_{j\in [q]} |\ch(j)| = |e(\Gcal)|$ indeterminates. The conditions are independent for each $j$. Namely, the block has the form
\[
M = \left(
\begin{array}{c|c|c|c}
    M[1] & 0 & \cdots & 0 \\
    \hline
    0 & M[2] & \cdots & 0 \\
    \hline
    \vdots & \cdots & \ddots & \vdots \\
    \hline
    0 & \cdots & 0 & M[q]
\end{array}
\right).
\]
Each sub-block has size $|\de(j)\setminus\ch(j)| \times |\ch(j)|$ and defines $M[j] \cdot \left( \lambda^{(0)}_{i,j} \right)_{i \in \ch(j)} = b[j]$ where
\begin{equation*}
\begin{aligned}
    M[j] &= 
    \left(
    \left( A^{(k)} - A^{(0)}\right)_{1,i}
    \right),
    \;
    k \in \de(j)\setminus\ch(j), \; i \in \ch(j), \\
    b[j] &= 
    \left(
    \left( A^{(k)} - A^{(0)}\right)_{1,j}
    \right),
    \;
    k \in \de(j)\setminus\ch(j). \\
\end{aligned}
\end{equation*}
At this point, it seems that the matrices defining the linear system depend on $F$ and $\Lambda^{(k)}$. However, following the proof of Proposition \ref{prop:difference_lambdas_rank1}, we have
\begin{equation*}
    \left( A^{(k)} - A^{(0)}\right)_{1,i} = \sum_{n \in [q]} f_{1, n} \: \Delta^{(k)}_{n, i} = \sum_{n \in \de(k)} f_{1, n} (I-\Lambda^{(0)})^{-1}_{n, k} \: \Delta^{(k)}_{k, i} = A^{(0)}_{1,k} \: \Delta^{(k)}_{k, i}.
\end{equation*}
Assuming $A^{(0)}_{1,k} \neq 0$ for every $k \in [q]$, which holds generically, we can rescale to obtain
\begin{equation}\label{eq:Mb}
\begin{aligned}
    M[j] &= 
    \left( \; \Delta^{(k)}_{k,i} \; \right),
    \;
    k \in \de(j)\setminus\ch(j), \; i \in \ch(j), \\
    b[j] &= 
    \left( \; \Delta^{(k)}_{k,j} \; \right),
    \;
    k \in \de(j)\setminus\ch(j),
\end{aligned}
\end{equation}
where $\Delta^{(k)} = (I-\Lambda^{(k)})^{-1} - (I-\Lambda^{(0)})^{-1}$. From this, we see that $M[j]$ and $b[j]$ depend only on the latent DAG and its parameters: the linear system \eqref{eq:linear_sys_no_ch} becomes
\begin{equation}\label{eq:lin_sys_reduction}
\sum_{i \in \ch(j) } \Delta^{(k)}_{k,i} \lambda^{(0)}_{i,j} = \Delta^{(k)}_{k,j},
\end{equation}
for all $j\in [q]$.
We compute the dimension of the solution space by comparing the ranks $|\de(j)\setminus\ch(j)| \times |\ch(j)|$ matrix $M[j]$ and the $|\de(j)\setminus\ch(j)| \times (|\ch(j)| + 1)$ matrix $(M[j]|b[j])$.
\begin{proposition}\label{prop:dim_ideal_I_any_pq}
    Assume that the interventions are soft. For each node $j$, let
    \[
    c_j = \begin{cases}
        -1 & \hbox{ if } \rank M[j] \neq \rank \left( M[j] | b[j] \right), \\
        |\ch(j)|-\rank M[j] & \hbox{ otherwise,}
    \end{cases}
    \]
    where $M[j]$ and $b[j]$ are defined in \eqref{eq:Mb}. Then, the ideal $\mathcal{I}$ in \eqref{eq:def_ideal} has dimension
    \[
    \dim \mathcal{I} = 
    \begin{cases}
        -1 & \hbox{ if } c_j = -1 \hbox{ for some } j\in [q], \\
        \sum_{j=1}^q c_j & \hbox{ otherwise}.
    \end{cases}
    \]
\end{proposition}
\begin{proof}
    The dimension of $\mathcal{I}$ is the dimension of the solution space of \eqref{eq:lin_sys_reduction}, that is, 
    \[
    M[j] \cdot \left( \lambda^{(0)}_{i,j} \right)_{i \in \ch(j)} = b[j]
    \]
    for all $j \in [q]$, by \eqref{eq:matrix_lin_sys_q>p}. Its dimension $c_j$ is $|\ch(j)|-\rank M[j]$ if $\rank M[j] \neq \rank \left( M[j] | b[j] \right)$. Otherwise, the solution space is empty and we set $c_j = -1$, as is convention.
\end{proof}

\begin{corollary}\label{coroll:identifiable_soft}
    With one soft intervention per latent node it is never possible to recover uniquely all the parameters of the model.
\end{corollary}
\begin{proof}
The result remains true if we assume knowledge of the latent DAG $\Gcal$.
    Let $\ch(i) = \emptyset$. Take $j\in \pa(i)$ such that $\an(j)\setminus\ch(j) = \emptyset$. Then $M[j] = \emptyset$, so $c_j = |\ch(j)|\geq 1$. Therefore, $\dim \mathcal{I}\geq 1$ and it is not possible to identify uniquely the parameters $f_{\ell,j}$ and $\lambda^{(k)}_{k,j}$, for $\ell \in [p]$ and $k \in \ch(j)\setminus \{i\}$.
\end{proof}
Adding interventions does not affect the matrices $M[j]$ in the proof of Corollary \ref{coroll:identifiable_soft}. Therefore Corollary \ref{coroll:identifiable_soft_more} follows: non-identifiability holds regardless of the number of interventions.

When $c_j = 0$ it is possible to identify uniquely all parameters $\lambda^{(0)}_{i,j}$ for $i\in\ch(j)$, as well as $f_{\ell,j}$ and $\lambda^{(k)}_{k,j}$, for $\ell \in [p]$ and $k \in \ch(j)$. The condition $c_j = 0$ holds, for example, when $\ch(j) = \{i_1\}$ and $\de(j) = \{i_1,i_2\}$.

\begin{example}
    We continue Example \ref{ex:graph4nodes_lin_ideal}. The matrices are
    \[
    M[1] = M[2] = M[3] = b[1] = b[2] = b[3] = \emptyset, \quad 
    M[4] = \begin{pmatrix}
        -14 & 5
    \end{pmatrix}, \: b[4] = -105,
    \]
    hence $c_1 = |\ch(1)| = 0$, $c_2 = |\ch(2)| = 1$, $c_3 = |\ch(3)| = 1$, $c_4 = |\ch(4)| - \rank M[4] = 2-1 = 1$. By Proposition \ref{prop:dim_ideal_I_any_pq}, we have $\dim \mathcal{I} = 3$ and in fact it is minimally generated by the $21$ linear polynomials in $24$ indeterminates in Example \ref{ex:graph4nodes_lin_ideal}.
\end{example}

The rank of $M[j]$ depends on the structure of the DAG $\Gcal$ beyond the number of children and descendants of $j$. This is highlighted in the following example. 

\begin{example}
    Consider the DAG
    \begin{center}
    \begin{tikzpicture}
    \node at (-3.2,0) {$\Gcal \; = $};
    \begin{scope}[every node/.style={circle,thick,draw}]
        \node (5) at (-2,0) {$5$};
        \node (4) at (0,0) {$4$};
        \node (3) at (2,0) {$3$};
        \node (2) at (4,0) {$2$};
        \node (1) at (6,0) {$1$};
    \end{scope}
    \begin{scope}[>={Stealth[PineGreen]},
                  every node/.style={fill=white,circle},
                  every edge/.style={draw=PineGreen,very thick}]
        \path [->] (5) edge (4);
        \path [->] (4) edge (3);
        \path [->] (3) edge (2);
        \path [->] (5) edge[bend right = 30] (3);
        \path [->] (3) edge[bend right = -30] (1);
    \end{scope}
    \end{tikzpicture}
    \end{center}
    Node $j=5$ has $2$ children and $4$ descendants, hence \eqref{eq:lin_sys_reduction} consists of equations
    \begin{equation}\label{eq:example_lin_sys}
        \sum_{i = 3,4} \Delta^{(k)}_{k,i} \lambda^{(0)}_{i,5} = \Delta^{(k)}_{k,5} \qquad k = 1,2.
    \end{equation}
    They impose conditions on the $\lambda_{i,j}^{(0)}$. There are two $\lambda_{i,j}^{(0)}$ from node $5$, namely $\lambda^{(0)}_{3,5}, \lambda^{(0)}_{4,5}$, and two equations. However, the equations in \eqref{eq:example_lin_sys} are dependent. We have
    \[
    M[5] = 
    \begin{pmatrix}
        \lambda^{(1)}_{1,3} - \lambda^{(0)}_{1,3} & (\lambda^{(1)}_{1,3} - \lambda^{(0)}_{1,3}) \lambda^{(0)}_{3,4} \\
        \lambda^{(2)}_{2,3} - \lambda^{(0)}_{2,3} & (\lambda^{(2)}_{2,3} - \lambda^{(0)}_{2,3}) \lambda^{(0)}_{3,4}
    \end{pmatrix}, 
    \qquad
    b[5] = \begin{pmatrix}
        (\lambda^{(1)}_{1,3} - \lambda^{(0)}_{1,3}) (\lambda^{(0)}_{3,4} \lambda^{(0)}_{4,5} + \lambda^{(0)}_{3,5}) \\
        (\lambda^{(2)}_{2,3} - \lambda^{(0)}_{2,3}) (\lambda^{(0)}_{3,4}\lambda^{(0)}_{4,5} + \lambda^{(0)}_{3,5})
    \end{pmatrix}, 
    \]
    so $\rank M[5] = \rank \left( M[5] | b[5] \right) = 1 < 2$. Hence we cannot recover the parameters $\lambda^{(0)}_{3,5}, \lambda^{(0)}_{4,5}$. The reason $\rank M[5] <2$ is that all the paths from $4$ to $1$ or $2$ (encoded in the second column of $M[5]$) and all the paths from $5$ to $1$ or $2$ (encoded in $b[5]$) go through $3$. This factorization of paths creates dependencies in $M[5], b[5]$, preventing identifiability.
\end{example}
To recover as many parameters as possible, DAGs should balance between too many children, hence too many indeterminates, and too few children, hence paths factorize more easily.

From an algorithmic point of view, we can check the rank condition in Proposition \ref{prop:dim_ideal_I_any_pq}. Indeed, we have matrices $A^{(k)}$, and we can compute, for all $i, k \in [q]$ with $i\neq k$, the entries 
\[
\Delta^{(k)}_{k,i} = \frac{A^{(k)}_{1,i} - A^{(0)}_{1,i}}{A^{(0)}_{1,k}}.
\]

\begin{remark}\label{rmk:q<p_sec3_soft}
    If $q\leq p$, the computations can be simplified. We can write \eqref{eq:lambdak_lin_sys} as
    \[
    \Lambda^{(k)} = (A^{(k)})^+ A^{(0)} \Lambda^{(0)} + I - (A^{(k)})^+ A^{(0)}.
    \]
    This writes the $\lambda^{(k)}_{k,j}$ indeterminates in terms of $\lambda^{(0)}_{i,j}$ indeterminates, and enables us to find the linear conditions on the $\lambda^{(0)}_{i,j}$ indeterminates. The reduction of the linear system is the same as in \eqref{eq:Mb}, as can be proved by noticing that for any $k\neq i$, the $k$-th row of $(A^{(k)})^+ A^{(0)}$ equals the $k$-th row of $\Delta^{(k)}$ up to sign. Indeed,
    \begin{align*}
        \left( (A^{(k)})^+ A^{(0)}\right)_{k,i} &= \left( (I-\Lambda^{(k)}) (I-\Lambda^{(0)})^{-1} \right)_{k,i} \\
        &= \sum_{\ell \in \pa(k)} -\lambda^{(k)}_{k,\ell} (I-\Lambda^{(0)}_{\ell,i})^{-1} + (I-\Lambda^{(0)})^{-1}_{k,i} \\
        &= - (I-\Lambda^{(k)})^{-1}_{k,i} + (I-\Lambda^{(0)})^{-1}_{k,i} = - \Delta^{(k)}_{k,i},
    \end{align*}
    where the last row used $(I-\Lambda^{(0)}_{\ell,i})^{-1} = (I-\Lambda^{(k)}_{\ell,i})^{-1}$ for every $\ell\in \pa(k)$.
\end{remark}

\subsubsection{Identifiability of the latent DAG}

For perfect interventions, Theorem \ref{thm_intro_perfect} shows that we can recover the latent DAG of the model, as well as the parameters. With soft interventions, we cannot recover the whole DAG and all the parameters (see Corollary \ref{coroll:identifiable_soft}). It is natural to wonder to what extent we can recover the latent DAG. 
Thanks to Proposition \ref{prop:dim_ideal_I_any_pq}, we can turn this into $2q$ rank computations.

\begin{definition}
    Given a model with matrices $A^{(k)}\in\R^{p\times q}$ for  $k = 0,\ldots,q$, we say that a DAG $\Gcal'$ is \emph{compatible} with the model if there exist parameters $F\in \R^{p\times q}$, $\Lambda^{(k)} \in \R^{q\times q}$ defined according to the latent DAG $\Gcal'$, such that $A^{(k)}=F(I-\Lambda^{(k)})^{-1}$ for all $k$.
\end{definition}

    If the true latent DAG is $\Gcal$, a compatible DAG $\Gcal'$ must satisfy $\overline{\Gcal} = \overline{\Gcal'}$, by Corollary~\ref{cor:transitive_closure}. 
To emphasize that a matrix such as $\Lambda^{(k)}$ depends on a DAG $\Gcal$, we write $\Lambda^{(k)}_{\Gcal}$. 
In the same spirit, we define
\begin{equation}\label{eq:Mb_GG'}
\begin{aligned}
    M_{\Gcal,\Gcal'}[j] &= \left( \; (\Delta^{(k)}_{\Gcal})_{k,i} \; \right) \; k \in \de_{\Gcal'}(j)\setminus \ch_{\Gcal'}(j), \; i \in \ch_{\Gcal'}(j), \\
    b_{\Gcal,\Gcal'}[j] &= \left( \; (\Delta^{(k)}_{\Gcal})_{k,j} \; \right) \; k \in \de_{\Gcal'}(j)\setminus \ch_{\Gcal'}(j),
\end{aligned}
\end{equation}
where the indexing depends on $\Gcal'$ and $\Delta^{(k)}_{\Gcal} = (I-\Lambda_{\Gcal}^{(k)})^{-1} - (I-\Lambda_{\Gcal}^{(0)})^{-1}$. 
Recall that
\[
[(I-\Lambda^{(0)})^{-1}]_{\mathscr{r},\mathscr{c}}
\]
denotes the submatrix with rows in $\mathscr{r}\subset [q]$ and columns in $\mathscr{c}\subset [q]$.
We give the following definition, already mentioned in the introduction.

\begin{definition}\label{def:softG}
    Given a DAG $\Gcal$ on $q$ nodes, its \emph{soft-compatible} class is
    \begin{multline}\label{eq:softG}
        \soft(\Gcal) = \Big\lbrace \Gcal' \hbox{ DAG} \;|\; \overline{\Gcal'} = \overline{\Gcal} \hbox{ and for all } j \in [q] \\ \rank [(I-\Lambda_{\Gcal}^{(0)})^{-1}]_{\de_{\Gcal'}(j)\setminus\ch_{\Gcal'}(j),\,\ch_{\Gcal'}(j)} = \rank [(I-\Lambda_{\Gcal}^{(0)})^{-1}]_{\de_{\Gcal'}(j)\setminus\ch_{\Gcal'}(j),\,\overline{\ch}_{\Gcal'}(j)} \; \Big\rbrace,
    \end{multline}
    where
    $\overline{\ch}(j) := \ch(j) \cup \{j\}$ for $j\in [q]$.
\end{definition}

The soft-compatible class of a DAG $\Gcal$ is the set of all graphs that are compatible with a model with latent DAG $\Gcal$, as follows.
\begin{theorem}\label{thm:soft-compatible}
    Consider LCD under under Assumption~\ref{assumption:main} with DAG $\Gcal$. Then a graph $\Gcal'$ is compatible with the model if and only if $\Gcal' \in \soft(\Gcal)$.
\end{theorem}
\begin{proof}
    The ranks of the matrices in \eqref{eq:softG} are the same as the ranks of the matrices $M_{\Gcal,\Gcal'}[j]$ and $\left( M_{\Gcal,\Gcal'}[j] | b_{\Gcal,\Gcal'}[j] \right)$, since the first matrices can be obtained from the second by replacing each term $\lambda^{(k)}_{k,i}-\lambda^{(0)}_{k,i}$ with $\lambda^{(0)}_{k,i}$. By the genericity of the parameters, this does not affect the rank. The genericity assumption allows us to switch between the parameters of the model and abstract indeterminates without change.
    Therefore, the ranks of the matrices in \eqref{eq:softG} coincide with the ranks of $M[j]$ and $\left( M[j] | b[j] \right)$ from \eqref{eq:Mb}. 
    
    The DAG $\Gcal'$ is compatible with the model if and only if the corresponding ideal $\mathcal{I}$ has non-negative dimension. Here the indeterminates of the system defined by $\mathcal{I}$ are the entries of $F$ and of $\Lambda^{(k)}_{\Gcal'}$. By Proposition \ref{prop:dim_ideal_I_any_pq}, it is equivalent to $c_j \neq -1$ for all $j$, which holds if and only if $\rank M[j] = \rank \left( M[j] | b[j] \right)$ for all $j$. This is equivalent to the condition $\Gcal' \in \soft(\Gcal)$.
\end{proof}

\begin{proof}[Proof of Theorem \ref{thm:soft1}]
    The linearity follows from equation \eqref{eq:def_ideal}. The positive dimensionality follows from Corollary \ref{coroll:identifiable_soft}. The compatibility class of DAGs is Theorem \ref{thm:soft-compatible}.
\end{proof}

We investigate the concept of soft-compatible class. If $\Gcal'$ has the same transitive closure as $\Gcal$ and if $\ch_{\Gcal'}(j) \supset \ch_{\Gcal}(j)$ for all $j$, then $\Gcal' \in \soft(\Gcal)$. This is true because 
\[
[(I-\Lambda_{\Gcal}^{(0)})^{-1}]_{\de_{\Gcal}(j)\setminus\ch_{\Gcal}(j),\,\ch_{\Gcal}(j)} \quad \hbox{ and } \quad [(I-\Lambda_{\Gcal}^{(0)})^{-1}]_{\de_{\Gcal}(j)\setminus\ch_{\Gcal}(j),\,\overline{\ch}_{\Gcal}(j)}
\]
always satisfies the rank condition (since by construction a solution with DAG $\Gcal$ exists). 
However, the matrix $[(I-\Lambda_{\Gcal}^{(0)})^{-1}]_{\de_{\Gcal'}(j)\setminus\ch_{\Gcal'}(j),\,\ch_{\Gcal'}(j)}$ is obtained from $[(I-\Lambda_{\Gcal}^{(0)})^{-1}]_{\de_{\Gcal}(j)\setminus\ch_{\Gcal}(j),\,\ch_{\Gcal}(j)}$ by adding columns and deleting rows, hence the column indexed by $j$ remains in the span of the columns indexed by its children. Therefore, in order to exit the soft-compatible class of $\Gcal$, a DAG $\Gcal'$ with the same transitive closure must have enough more children at some node.

If $q = 2$, $\soft(\Gcal) = \Gcal$ for every DAG, since there are no distinct DAGs with the same transitive closure. For $q = 3$, the only case in which two DAGs have the same transitive closure (up to relabeling the nodes) is the segment DAG $Z_3 \to Z_2 \to Z_1$, denoted by $\Mns$, and the DAG with extra edge $Z_3 \to Z_1$, denoted by $\blacktriangle$. Since $\blacktriangle$ is obtained from $\Mns$ by adding $Z_1$ to the children of $Z_3$, we know that $\blacktriangle \in \soft(\Mns)$. On the other hand, since $\de_{\Mns}(j)\setminus\ch_{\Mns}(j) = \emptyset$ for $j=1,2$, the only relevant submatrices are obtained for $j=3$ and they are
\begin{gather*}
    [(I-\Lambda_{\blacktriangle}^{(0)})^{-1}]_{1,2} \quad \hbox{ and } \quad [(I-\Lambda_{\blacktriangle}^{(0)})^{-1}]_{1,\, \{2,3\}}  
\end{gather*}
for which the rank condition is trivially satisfied, hence $\Mns \in \soft(\blacktriangle)$. Therefore, for DAGs on $3$ nodes, the soft-compatible classes are the classes of DAGs with the same transitive closure.

For $q=4$, if every node has at most $1$ descendant that is not a child, the rank condition is always satisfied. The only time this might not hold is for DAGs whose transitive closure is the complete DAG on $4$ nodes. There are $8$ such DAGs (up to relabeling). We compute their soft-compatible classes.

\begin{example}\label{ex:soft_4_nodes}
    Consider DAGs on $4$ nodes, with transitive closure the complete DAG
    \begin{center}
    \begin{tikzpicture}
    \node at (-1,0) {$\Gcal_1 \; = $};
    \begin{scope}[every node/.style={circle,thick,draw}]
        \node (4) at (0,0) {$4$};
        \node (3) at (1.5,0) {$3$};
        \node (2) at (3,0) {$2$};
        \node (1) at (4.5,0) {$1$};
    \end{scope}
    \begin{scope}[>={Stealth[PineGreen]},
                  every node/.style={fill=white,circle},
                  every edge/.style={draw=PineGreen,very thick}]
=        \path [->] (4) edge (3);
        \path [->] (3) edge (2);
        \path [->] (2) edge (1);
    \end{scope}
    \node at (5.1,-0.1) {,};
    \end{tikzpicture}
    \quad
    \begin{tikzpicture}
    \node at (-1,0) {$\Gcal_2 \; = $};
    \begin{scope}[every node/.style={circle,thick,draw}]
        \node (4) at (0,0) {$4$};
        \node (3) at (1.5,0) {$3$};
        \node (2) at (3,0) {$2$};
        \node (1) at (4.5,0) {$1$};
    \end{scope}
    \begin{scope}[>={Stealth[PineGreen]},
                  every node/.style={fill=white,circle},
                  every edge/.style={draw=PineGreen,very thick}]
=        \path [->] (4) edge (3);
        \path [->] (3) edge (2);
        \path [->] (2) edge (1);
        \path [->] (3) edge[bend right = -30] (1);
    \end{scope}
    \node at (5.1,-0.1) {,};
    \end{tikzpicture}
    \end{center}
    \begin{center}
    \begin{tikzpicture}
    \node at (-1,0) {$\Gcal_3 \; = $};
    \begin{scope}[every node/.style={circle,thick,draw}]
        \node (4) at (0,0) {$4$};
        \node (3) at (1.5,0) {$3$};
        \node (2) at (3,0) {$2$};
        \node (1) at (4.5,0) {$1$};
    \end{scope}
    \begin{scope}[>={Stealth[PineGreen]},
                  every node/.style={fill=white,circle},
                  every edge/.style={draw=PineGreen,very thick}]
=        \path [->] (4) edge (3);
        \path [->] (3) edge (2);
        \path [->] (2) edge (1);
        \path [->] (4) edge[bend right = -30] (1);
    \end{scope}
    \node at (5.1,-0.1) {,};
    \end{tikzpicture}
    \quad
    \begin{tikzpicture}
    \node at (-1,0) {$\Gcal_4 \; = $};
    \begin{scope}[every node/.style={circle,thick,draw}]
        \node (4) at (0,0) {$4$};
        \node (3) at (1.5,0) {$3$};
        \node (2) at (3,0) {$2$};
        \node (1) at (4.5,0) {$1$};
    \end{scope}
    \begin{scope}[>={Stealth[PineGreen]},
                  every node/.style={fill=white,circle},
                  every edge/.style={draw=PineGreen,very thick}]
=        \path [->] (4) edge (3);
        \path [->] (3) edge (2);
        \path [->] (2) edge (1);
        \path [->] (4) edge[bend right = -30] (2);
    \end{scope}
    \node at (5.1,-0.1) {,};
    \end{tikzpicture}
    \end{center}
    \begin{center}
    \begin{tikzpicture}
    \node at (-1,0) {$\Gcal_5 \; = $};
    \begin{scope}[every node/.style={circle,thick,draw}]
        \node (4) at (0,0) {$4$};
        \node (3) at (1.5,0) {$3$};
        \node (2) at (3,0) {$2$};
        \node (1) at (4.5,0) {$1$};
    \end{scope}
    \begin{scope}[>={Stealth[PineGreen]},
                  every node/.style={fill=white,circle},
                  every edge/.style={draw=PineGreen,very thick}]
=        \path [->] (4) edge (3);
        \path [->] (3) edge (2);
        \path [->] (2) edge (1);
        \path [->] (3) edge[bend right = 30] (1);
        \path [->] (4) edge[bend right = -30] (1);
    \end{scope}
    \node at (5.1,-0.1) {,};
    \end{tikzpicture}
    \quad
    \begin{tikzpicture}
    \node at (-1,0) {$\Gcal_6 \; = $};
    \begin{scope}[every node/.style={circle,thick,draw}]
        \node (4) at (0,0) {$4$};
        \node (3) at (1.5,0) {$3$};
        \node (2) at (3,0) {$2$};
        \node (1) at (4.5,0) {$1$};
    \end{scope}
    \begin{scope}[>={Stealth[PineGreen]},
                  every node/.style={fill=white,circle},
                  every edge/.style={draw=PineGreen,very thick}]
=        \path [->] (4) edge (3);
        \path [->] (3) edge (2);
        \path [->] (2) edge (1);
        \path [->] (3) edge[bend right = 30] (1);
        \path [->] (4) edge[bend right = -30] (2);
    \end{scope}
    \node at (5.1,-0.1) {,};
    \end{tikzpicture}
    \end{center}
    \begin{center}
    \begin{tikzpicture}
    \node at (-1,0) {$\Gcal_7 \; = $};
    \begin{scope}[every node/.style={circle,thick,draw}]
        \node (4) at (0,0) {$4$};
        \node (3) at (1.5,0) {$3$};
        \node (2) at (3,0) {$2$};
        \node (1) at (4.5,0) {$1$};
    \end{scope}
    \begin{scope}[>={Stealth[PineGreen]},
                  every node/.style={fill=white,circle},
                  every edge/.style={draw=PineGreen,very thick}]
=        \path [->] (4) edge (3);
        \path [->] (3) edge (2);
        \path [->] (2) edge (1);
        \path [->] (4) edge[bend right = 30] (2);
        \path [->] (4) edge[bend right = -30] (1);
    \end{scope}
    \node at (5.1,-0.1) {,};
    \end{tikzpicture}
    \quad
    \begin{tikzpicture}
    \node at (-1,0) {$\Gcal_8 \; = $};
    \begin{scope}[every node/.style={circle,thick,draw}]
        \node (4) at (0,0) {$4$};
        \node (3) at (1.5,0) {$3$};
        \node (2) at (3,0) {$2$};
        \node (1) at (4.5,0) {$1$};
    \end{scope}
    \begin{scope}[>={Stealth[PineGreen]},
                  every node/.style={fill=white,circle},
                  every edge/.style={draw=PineGreen,very thick}]
=        \path [->] (4) edge (3);
        \path [->] (3) edge (2);
        \path [->] (2) edge (1);
        \path [->] (3) edge[bend right = 30] (1);
        \path [->] (4) edge[bend right = -25] (2);
        \path [->] (4) edge[bend right = -30] (1);
    \end{scope}
    \node at (5.1,-0.1) {.};
    \end{tikzpicture}
    \end{center}
    The submatrices of $(I-\Lambda_{\Gcal}^{(0)})^{-1}$ for $j= 1,2,3$ are either empty or they have one row, for every $\Gcal'$. Therefore, on these nodes the rank condition is always satisfied. The interesting rank condition comes from the submatrices for $j=4$. By computing these submatrices for all pairs of DAGs, one obtains 
    \begin{gather*}
        \soft(\Gcal_1) = \soft(\Gcal_2) = \{ \Gcal_1, \Gcal_2, \Gcal_3, \Gcal_4, \Gcal_5, \Gcal_6, \Gcal_7, \Gcal_8 \}, \\
        \soft(\Gcal_3) = \soft(\Gcal_4) = \soft(\Gcal_5) = \soft(\Gcal_6) = \soft(\Gcal_7) = \soft(\Gcal_8) = \{\Gcal_3, \Gcal_4, \Gcal_5, \Gcal_6, \Gcal_7, \Gcal_8 \}.
    \end{gather*}
    For $n = 1,2$ and $m = 3,\ldots,8$, we  have $\Gcal_n \not\in \soft(\Gcal_m)$, since 
    \begin{align*}
    \rank [(I-\Lambda_{\Gcal_m}^{(0)})^{-1}]_{\de_{\Gcal_n}(4)\setminus\ch_{\Gcal_n}(4),\,\ch_{\Gcal_n}(4)} &= \rank [(I-\Lambda_{\Gcal_m}^{(0)})^{-1}]_{\{1,2\},\,3} = 1, \\ 
    \rank [(I-\Lambda_{\Gcal_m}^{(0)})^{-1}]_{\de_{\Gcal_n}(4)\setminus\ch_{\Gcal_n}(4),\,\overline{\ch}_{\Gcal_n}(4)} &= \rank [(I-\Lambda_{\Gcal_m}^{(0)})^{-1}]_{\{1,2\},\, \{3,4\}} = 2. \qedhere
    \end{align*}
\end{example}
Soft-compatible classes are in general smaller than the DAGs with the same transitive closure. In a soft-compatible class a unique sparsest DAG does not exist, see Example \ref{ex:soft_4_nodes} where the sparsest DAGs in $\soft(\Gcal_8)$ are $\Gcal_3, \Gcal_4$. Moreover, soft-compatible classes are not equivalence classes, see Example \ref{ex:soft_4_nodes} where $\Gcal_8 \in \soft(\Gcal_1)$ but $\Gcal_1 \not\in \soft(\Gcal_8)$.

\subsubsection{More interventions}

So far in Section \ref{sec:recovery_interventions} we have mostly assumed one intervention per latent node, and we proved that this is not sufficient to recover the latent DAG or the parameters.
It is natural to wonder whether more interventions would allow the recovery. We already discussed after Corollary \ref{coroll:identifiable_soft} that complete identifiability is impossible regardless of the number of interventions; in other words, Corollary~\ref{coroll:identifiable_soft_more} holds.

In some cases, increasing the number of interventions allows us to recover more parameters (see Example \ref{ex:more_int_yes}). In other cases, more interventions do not improve the parameters that can be recovered (see Example \ref{ex:more_int_no}).

\begin{example}\label{ex:more_int_no}
    Consider the DAG
    \begin{center}
    \begin{tikzpicture}
    \node at (-1.2,0) {$\Gcal \; = $};
    \begin{scope}[every node/.style={circle,thick,draw}]
        \node (4) at (0,0) {$4$};
        \node (3) at (2,0) {$3$};
        \node (2) at (4,0) {$2$};
        \node (1) at (6,0) {$1$};
    \end{scope}
    \begin{scope}[>={Stealth[PineGreen]},
                  every node/.style={fill=white,circle},
                  every edge/.style={draw=PineGreen,very thick}]
=        \path [->] (4) edge (3);
        \path [->] (3) edge (2);
        \path [->] (2) edge (1);
        \path [->] (4) edge[bend right = -30] (2);
    \end{scope}
    \end{tikzpicture}
    \end{center}
Node $j=4$ has $2$ children and $1$ descendant, so $c_4 = 1$. We are interested in the question: if we have more interventions, can we recover more? In particular, can we recover $\lambda^{(0)}_{2,4}, \lambda^{(0)}_{3,4}$? 

    The only intervention that could give us more information is an extra intervention on node $1$, since that is the only descendant of $4$ which is not its child. Assume a fifth context also with intervention target $1$. We have a linear system with two equations and two unknowns:
    \begin{align*}
        \Delta^{(1)}_{1,2}\lambda^{(0)}_{2,4} + \Delta^{(1)}_{1,3}\lambda^{(0)}_{3,4} &= \Delta^{(1)}_{1,4}, \\
        \Delta^{(5)}_{1,2}\lambda^{(0)}_{2,4} + \Delta^{(5)}_{1,3}\lambda^{(0)}_{3,4} &= \Delta^{(5)}_{1,4}.
    \end{align*}
    The matrix of this linear system, however, has rank $1$ because
    \[
    \begin{pmatrix}
        \Delta^{(1)}_{1,3} \\
        \Delta^{(5)}_{1,3}
    \end{pmatrix}
    = \lambda^{(0)}_{2,3}
    \begin{pmatrix}
        \Delta^{(1)}_{1,2} \\
        \Delta^{(5)}_{1,2}
    \end{pmatrix}
    \qquad \hbox{and} \qquad
    \begin{pmatrix}
        \Delta^{(1)}_{1,4} \\
        \Delta^{(5)}_{1,4}
    \end{pmatrix}
    = (\lambda^{(0)}_{2,4} + \lambda^{(0)}_{2,3} \lambda^{(0)}_{3,4})
    \begin{pmatrix}
        \Delta^{(1)}_{1,2} \\
        \Delta^{(5)}_{1,2}
    \end{pmatrix}.
    \]
    All paths from $4$ to $1$ must go through $2$ and this means no additional parameters can be recovered from the extra intervention.
\end{example}

\begin{example}\label{ex:more_int_yes}
    Consider the DAG
    \begin{center}
    \begin{tikzpicture}
    \node at (-1.2,0) {$\Gcal \; = $};
    \begin{scope}[every node/.style={circle,thick,draw}]
        \node (4) at (0,0) {$4$};
        \node (3) at (2,0) {$3$};
        \node (2) at (4,0) {$2$};
        \node (1) at (6,0) {$1$};
    \end{scope}
    \begin{scope}[>={Stealth[PineGreen]},
                  every node/.style={fill=white,circle},
                  every edge/.style={draw=PineGreen,very thick}]
=        \path [->] (4) edge (3);
        \path [->] (3) edge (2);
        \path [->] (2) edge (1);
        \path [->] (4) edge[bend right = -30] (2);
        \path [->] (3) edge[bend right = 30] (1);
    \end{scope}
    \end{tikzpicture}
    \end{center}
    and focus on the node $j=4$. The linear system is as above, but now
    \[
    \begin{pmatrix}
        \Delta^{(1)}_{1,3} \\
        \Delta^{(5)}_{1,3}
    \end{pmatrix}
    \nparallel
    \begin{pmatrix}
        \Delta^{(1)}_{1,2} \\
        \Delta^{(5)}_{1,2}
    \end{pmatrix},
    \]
    so the matrix is full rank and we can recover the parameters $\lambda^{(0)}_{2,4}, \lambda^{(0)}_{3,4}$.
\end{example}

To conclude, more interventions do not necessarily allow us to recover the DAG or the parameters. Indeterminates $\lambda^{(k)}_{i,j}$ with $\ch(i) = \emptyset$ are not the only ones that cannot be recovered regardless of how many soft interventions we have - this is also the case for $\lambda^{(k)}_{2,4}$ and $\lambda^{(k)}_{3,4}$ in Example \ref{ex:more_int_no}. But there are examples where more interventions reduce the dimension of the solution space - we can recover $\lambda^{(k)}_{2,4}$ and $\lambda^{(k)}_{3,4}$ in Example~\ref{ex:more_int_yes}.

\section{Algorithm}\label{sec:algorithms}

Coupled tensor decomposition of higher-order cumulants enables us to recover parameters in an LCD model. 
In this section, we explain how to turn our results into a numerical algorithm for LCD. 
The algorithm has two main steps. The first (see Section~\ref{subsec:algo_any_pq}) is to identify intervention targets, permutation, and scaling. This turns the results of Section~\ref{sec:tensor_dec} into an algorithm, and works for both perfect and soft interventions. The second step (see Section~\ref{subsec:parameters}) is to recover the parameters of the model. We do this step for perfect interventions (following Theorem~\ref{thm_intro_perfect}) and not for soft interventions, in light of Theorem~\ref{thm:soft1}. Both steps simplify when $q \leq p$; that is, when the number of latent variables is at most the number of observed variables, see Section~\ref{subsec:injective}. We test our algorithms on synthetic data in Section~\ref{subsec:experiments}.  

\subsection{Recovery of intervention targets, permutation, and scaling}\label{subsec:algo_any_pq}

The algorithm input is the $d$-th order cumulants $\kappa_d(X^{(k)})$ as in~\eqref{eqn:cumulant_decomp2}, for $k \in K \cup \{0 \}$ ranging over contexts and a fixed $d \geq 3$. 
The input tensors are either exact (population cumulants) or approximate (sample cumulants). 

For our fixed $d$, we assume 
that the decomposition in~\eqref{eqn:cumulant_decomp2} is unique, that $\kappa_d(\epsilon^{(k)}_i)\neq 0$ for all $k$ and all $i \in [q]$, and that $\kappa_d(\epsilon^{(i_k)}_{i_k}) \neq \pm1$ for all $k$.
This is the same assumption as appears temporarily 
in the proof of Proposition \ref{prop:with_DP}.
The assumption is not necessary,
since one can combine information from multiple higher-order cumulants, 
but it helps our exposition, as in the proof of Proposition \ref{prop:with_DP}. 
In our experiments, we consider $d=3$ and $d=4$.

Tensor decomposition recovers the matrices $A^{(k)}$ up to permutation and scaling, by Proposition~\ref{prop:non_zero_coefficients}. 
Tensor decomposition thus recovers a set of matrices $\{ A^{(k)} D^{(k)} P^{(k)} : k \in K \cup \{ 0 \} \}$ for unknown scaling matrices $D^{(k)}$ and unknown permutations $P^{(k)}$.
In practice, any numerical tensor decomposition algorithm can be used for this step. We use the subspace power method~\cite{kileel2019subspace} or simultaneous diagonalization~\cite{harshman70} when $q \leq p$. 
We can assume without loss of generality that $D^{(0)} = P^{(0)} = I$, see Proposition \ref{prop:wlog_i}. Then the other scaling matrices $D^{(k)}$ have all entries $\pm 1$ except one, by~\eqref{eqn:formula_for_D} of Proposition~\ref{prop:with_DP}.
There is one entry that is not $\pm1$, which corresponds to the intervention target of context $k$.

To find the permutation, we consider the difference $A^{(0)} - A^{(k)} D^{(k)} P^{(k)}$, as $D^{(k)}$ varies over diagonal matrices with diagonal $\pm 1$ and $P^{(k)}$ varies over permutation matrices. 
That is, the product $D^{(k)} P^{(k)}$ is a signed permutation matrix.
The rank of the difference equals one if and only if we have the correct sign and order, see Corollary \ref{coroll:rank1}. 
This suggests an algorithm: for all $q \times q$ signed permutation matrices $Q$ (of which there are $2^q \times q!$) compute $A^{(k)} D^{(k)} P^{(k)} Q$ and compute the  second largest eigenvalue of the matrix 
$A^{(0)} - A^{(k)} D^{(k)} P^{(k)} Q$.
Choose $Q$ for which this eigenvalue is smallest. This $Q$ is $(P^{(k)})^\top$, up to sign, where the sign gives all but entry  $D^{(k)}_{i_k,i_k}$ of $D^{(k)}$. See Algorithm \ref{alg:permutation}.

To find the remaining entry of $D^{(k)}$, we compare the columns of $A^{(0)}$ and $A^{(k)} D^{(k)}$. The only column that differs between the two matrices is the $i_k$-th column. The $i_k$-th columns of the two matrices are collinear, with scaling $D^{(k)}_{i_k,i_k}$. This recovers the intervention target $i_k$ and the scaling matrix $D^{(k)}$. See Algorithm \ref{alg:target}.

A faster way approach to find the intervention targets, permutation, and scaling could be to implement Proposition \ref{prop:3cases_target_permutation}. One  can compare each column of $A^{(0)}$ to each column of $A^{(k)} D^{(k)} P^{(k)}$, e.g. by projecting a column $\mathbf{v}_1$ of one matrix onto another column $\mathbf{v}_2$ of the other. Each time, the residue of the projection $\| \mathbf{v}_1 - \pi_{\langle \mathbf{v}_2 \rangle}(\mathbf{v}_1)\|$ and the scaling $\frac{1}{\| \mathbf{v}_2\|} \|\pi_{\langle \mathbf{v}_2\rangle }(\mathbf{v}_1)\|$ can be stored and thresholds can be used to decide which numerical values are zero or $\pm1$. Such a procedure is numerically sensitive and influenced by the threshold. 
The threshold determines the assignment of intervention target, and we want the $|K|$ intervention targets to cover all latent nodes. One must choose a threshold that gives such an assignment. We leave this for future work.

\subsection{Recovery of parameters}
\label{subsec:parameters}

Once the intervention targets, permutation, and scaling are recovered, using Section~\ref{subsec:algo_any_pq}, we have matrices $A^{(k)} = F(I-\Lambda^{(k)})^{-1}$ for $k = 0,\ldots,q$. We can relabel contexts so that the intervention target of the $k$-th context is $k$. We construct $(I-\Lambda^{(0)})^{-1}$ as 
\begin{equation*}
    (I-\Lambda^{(0)})^{-1}_{i,j} = 
    \begin{cases}
        1 & i = j, \\
        \frac{A^{(0)}_{1,j} - A^{(i)}_{1,j}}{A^{(0)}_{1,i}} & i \neq j,
    \end{cases}
\end{equation*}
following the proof of Theorem \ref{thm_intro_perfect}.
We invert this matrix to find $\Lambda^{(0)}$. This is Algorithm \ref{alg:Lambda}. Finally, we recover $F$ using 
$F = A^{(0)} (I-\Lambda^{(0)})$.
One can compare the products $A^{(k)} (I-\Lambda^{(k)})$ for different contexts $k$ to test the goodness of fit of the LCD model. In theory, these should all return the same mixing matrix $F$.

\subsection{The injective case}\label{subsec:injective}

Restricting to the case $q\leq p$ allows for simplifications, cf. Remarks \ref{rmk:q<p_sec2}, \ref{rmk:q<p_sec3_perfect}, and \ref{rmk:q<p_sec3_soft}. 
First, the tensor decomposition step can achieved using simultaneous diagonalization~\cite{harshman70}.
We explain how to recover the intervention targets, permutation, and scaling in this setting.
When $q\leq p$ the Moore-Penrose pseudo-inverse satisfies 
\begin{equation}
\label{eqn:C-prod}
        C^{(k)} := \left( F (I-\Lambda^{(k)})^{-1} D^{(k)} P^{(k)} \right)^+ = (P^{(k)})^\top (D^{(k)})^{-1} (I-\Lambda^{(k)}) H,
\end{equation}
where $H = F^+$.
In particular, $C^{(0)} = (I-\Lambda^{(0)}) H$. 
Let $(\mathbf{c}^{(k)})^\ell$ denote the $\ell$-th row of $C^{(k)}$. Then we have the following result, in the same spirit as Proposition \ref{prop:3cases_target_permutation}.

\begin{proposition}\label{prop:int_target}
    Consider LCD under Assumption~\ref{assumption:main} where $q\leq p$. Fix $k\in K$ and let $\sigma$ be the permutation associated to the permutation matrix $P^{(k)}$. Then
    \[
    (\mathbf{c}^{(0)})^\ell = (\mathbf{c}^{(k)})^{\sigma(\ell)}
    \]
    if and only if $\ell \neq i_k$, where $i_k$ is the target of the $k$-th intervention. 
\end{proposition}
\begin{proof}
    We have formulae
    \[
    (\mathbf{c}^{(0)})^\ell = \mathbf{h}^\ell-\sum_{j\in \text{pa}(\ell)}\lambda_{\ell, j}^{(0)} \mathbf{h}^j,\qquad 
    (\mathbf{c}^{(k)})^{\sigma(\ell)} = \frac{1}{D^{(k)}_{\ell, \ell }} \left(\mathbf{h}^\ell -\sum_{j\in \text{pa}(\ell )}\lambda^{(k)}_{\ell, j}\mathbf{h}^j\right),
    \]
    by~\eqref{eqn:C-prod}.
    If $\ell \neq i_k$, then $D_{\ell, \ell}^{(k)} = 1$ and $\lambda_{\ell, j}^{(k)} = \lambda_{\ell, j}^{(0)}$ for all $j \in [q]$. Hence these two expressions coincide. If $\ell = i_k$ then under the genericity assumption we have $(\mathbf{c}^{(0)})^\ell \neq (\mathbf{c}^{(k)})^{\sigma(\ell)}$.
\end{proof}
Proposition~\ref{prop:int_target} enables us to find the intervention target $i_k$ and the permutation matrix $P^{(k)}$, as follows. 
For sufficiently general parameters $\Lambda^{(0)}$ and $\Lambda^{(k)}$ and error distribution $\epsilon^{(k)}$, the rows $({\bf c}^{(0)})^{i_k}$ and $({\bf c}^{(k)})^{\sigma(i_k)}$ differ. Hence, $i_k$ is the index of the row of $C^{(0)}$ without a match in $C^{(k)}$. We recover $P^{(k)}$ by matching the remaining rows: if $\ell \neq i_k$, then 
$P^{(k)}_{\ell, j}=1$, where
$(\mathbf{c}^{(0)})^\ell = (\mathbf{c}^{(k)})^j$. This finds all but one row of $P^{(k)}$. It has a unique completion to a permutation matrix: $P^{(k)}_{i_k, j}=1$, where $(\mathbf{c}^{(k)})^j$ is the row of $C^{(k)}$ without a match in $C^{(0)}$. 
See Algorithms \ref{alg:intervention_target_q<p} and \ref{alg:permutation_q<p}.
Algorithm \ref{alg:scaling_q<p} recovers the scalings $D^{(k)}$.

We now explain how to recover the parameters.
In light of the above, we have matrices $(A^{(k)})^+ = (I-\Lambda^{(k)}) H$. We find $H = F^+$ as follows. The $i_k$-th row of $(A^{(k)})^+$ has entries\[
(A^{(k)})^+_{i_k,j} = \sum_{\ell \in [q]} (I-\Lambda^{(k)})_{i_k,\ell} H_{\ell,j} = H_{i_k,j}.
\]
This is Algorithm \ref{alg:F_q<p}. To find $\Lambda^{(0)}$, we use
\[
(A^{(0)})^+ H^+ = (I-\Lambda^{(0)}) H H^+ = (I-\Lambda^{(0)}).
\]
Following the initial tensor decomposition, the time complexity of this algorithm is determined by the time required for the alignment and to calculate the pseudo-inverses of the products. The former takes time $O(q^3p)$ and the latter $O(p^2q^2)$, so the overall runtime is $O(q^2p\max(p,q))$.
This improves on the algorithm in  \cite{SSBU23:LinearCausalDisentanglementInterventions} provided 
we ignore the time taken to construct and decompose the higher-order cumulants.

\subsection{Numerical experiments}\label{subsec:experiments}

We test our algorithms on synthetic data. 
The general procedure for any $(p,q)$ is implemented in Algorithms \ref{alg:permutation}-\ref{alg:Lambda} in Appendix~\ref{sec:Appendix} and the injective case ($q\leq p$) is implemented in Algorithms \ref{alg:intervention_target_q<p}-\ref{alg:F_q<p}. 
For the general setting, we use $d=4$, since we use the subspace power method for tensor decomposition and it requires input of even order. 
For the injective case, we study $d=3$.  

We sample graphs using the Python package \verb|causaldag| \cite{squires2018causaldag}. It extends the Erd\H{o}s–R\'enyi model \cite{renyi1959random} to DAGs: given an edge density $\rho$, the edge $i\rightarrow j$ is added to the graph with probability $\rho$, and if and only if $i>j$. We fix $\rho=0.75$, sample the entries of $F^+$ independently from \texttt{Unif}$([-2,2])$, and the non-zero entries of $\Lambda^{(0)}$ independently from \texttt{Unif}$(\pm[0.25,1])$,
as in \cite{SSBU23:LinearCausalDisentanglementInterventions}. We fix $p=5$ and vary $q$ from $2$ to $7$. For each value of $q$, we generate $500$ models and calculate the relative Frobenius error for the recovery of $F$ and $\Lambda^{(0)}$, which is
$\frac{1}{\|M\|}\|  \widetilde{M} - M \|$,
where $\|\cdot \|$ denotes the Frobenius norm, $M$ is the true matrix, and $\widetilde{M}$ is the recovered matrix. We also calculate DAG recovery error, as follows. A penalty of $1$ is incurred if the algorithm recovers a non-existent edge or misses an existing one, while recovering an edge in the wrong direction incurs a penalty of $2$. Then we sum the penalty over all edges. 

\begin{figure}[!ht]
    \centering
    \includegraphics[width = 0.49\textwidth]{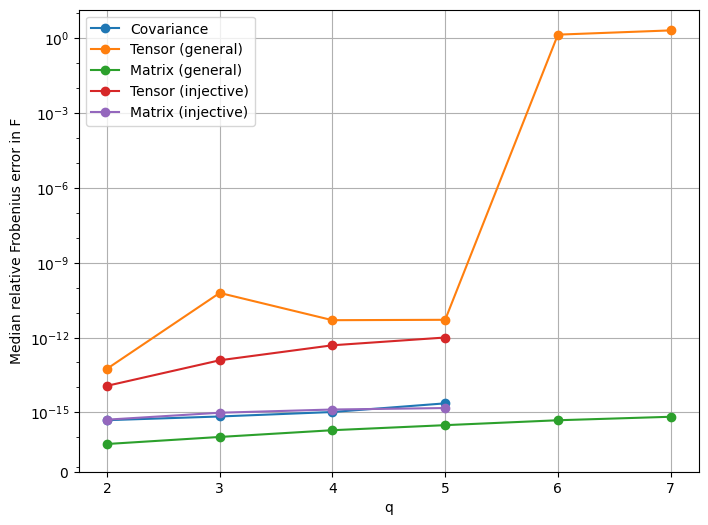}
    \:
    \includegraphics[width = 0.49\textwidth]{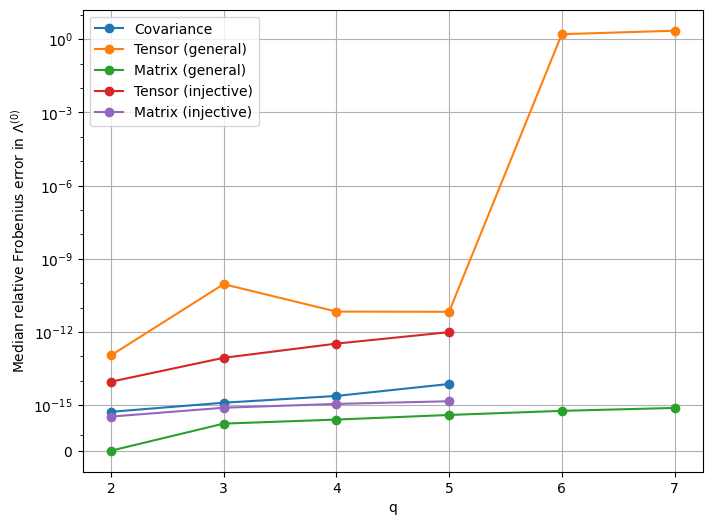}
    \caption{Median relative Frobenius error in the recovery of $F$ (left) and $\Lambda^{(0)}$ (right) when $p=5$. Note the logarithmic scale on the $y$-axis for all positive $y$-coordinates. The five algorithms are: (i) Covariance, the algorithm used in \cite{SSBU23:LinearCausalDisentanglementInterventions} to recover the parameters from the covariance matrices of $X^{(k)}$ (blue), (ii) Tensor (general), the general algorithm with cumulants as input (orange), (iii) Matrix (general), the general algorithm with factor matrices as input (green), (iv) Tensor (injective), the injective algorithm with cumulants as input (red), and (v) Matrix (injective), the injective algorithm with factor matrices as input (purple). For DAG recovery, all methods recovered the correct DAG every time, except the general tensor method when $q \geq 6$, which had a median DAG error of 3.6 for $q=6$ and 4.1 for $q=7$.} 
    \label{fig:recovery_F}
\end{figure}
We plot the median error in recovering $F$, $\Lambda^{(0)}$ in Figure \ref{fig:recovery_F}. A significant portion of the error is due to the tensor decomposition step, so we display the error of the recovery starting from the cumulants $\kappa_d(X^{(k)})$ as well as directly from the factor matrices $A^{(k)} D^{(k)} P^{(k)}$, as if these had been recovered perfectly from tensor decomposition. We used the same 500 models for 
algorithms (ii), (iii), (iv), and (v), but a different set of 500 models for algorithm (i), since model generation is part of the pipeline in \cite{SSBU23:LinearCausalDisentanglementInterventions}.

We can push our computation further when specializing to the injective case $q\leq p$, as shown in Figure \ref{fig:recovery_injective}. Here we fix $p=10$ and plot again the median error in recovering $F, \Lambda^{(0)}$, using the same approach as in the general case.
\begin{figure}
    \centering
    \includegraphics[width = 0.49\textwidth]{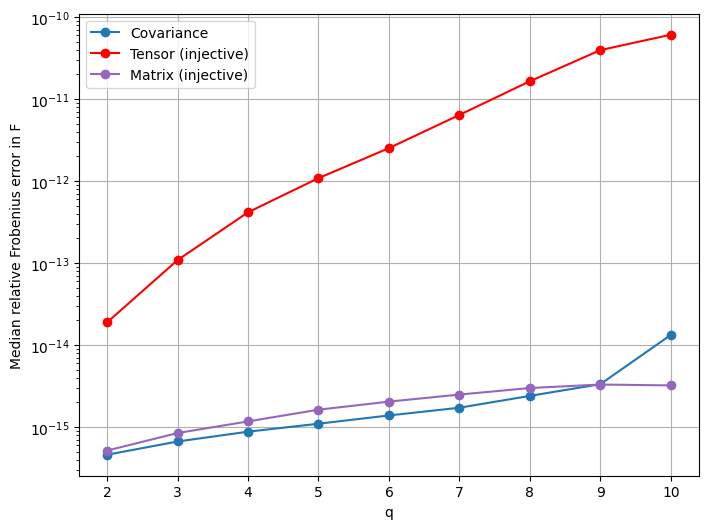}
    \:
    \includegraphics[width = 0.49\textwidth]{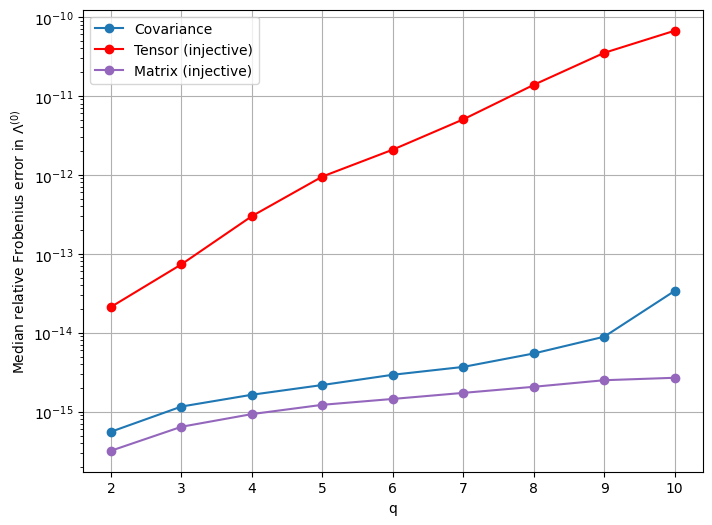}
    \caption{Median relative Frobenius error in the recovery of $F$ (left) and $\Lambda^{(0)}$ (right) when $p=10$ and $q\leq p$. Note the logarithmic scale on the $y$-axis for all positive $y$-coordinates. The three algorithms are: (i) Covariance, the algorithm used in \cite{SSBU23:LinearCausalDisentanglementInterventions} to recover the parameters from the covariance matrices of $X^{(k)}$ (blue), (ii) Tensor (injective), the injective algorithm with cumulants as input (red), and (iii) Matrix (injective), the injective algorithm with factor matrices as input (purple). For DAG recovery, all methods recovered the correct DAG every time.}
    \label{fig:recovery_injective}
\end{figure}

\section{Outlook}\label{sec:outlook}

We have studied the identifiability of linear causal disentanglement using tensor decomposition of higher-order cumulants. We view the parameters compatible with a given model as the solution space to a system of equations. Identifiability holds when the space has dimension zero, and can be achieved using perfect interventions. Here, we give an algorithm to recover the parameters. For soft interventions, we recover a compatibility class of graphs and parameters.
We conclude with some open problems for future investigation.

On the theoretical side, the first question is of combinatorial nature. The definition of $\soft(\Gcal)$ in Definition~\ref{def:softG} involves ranks of matrices. These rank conditions encode information about the paths in~$\Gcal$. This suggests the following problem.
\begin{problem}
    Find a combinatorial description of $\soft(\Gcal)$, based on the structure of $\Gcal$.
\end{problem}

Higher-order cumulants can reduce the degree of the solution space of parameters, as compared to covariance matrices, see Proposition \ref{prop:nonlinear_covariance}. An open problem is whether they restricts the set of compatible DAGs. 
\begin{problem}
    Let $\soft_2(\Gcal)$ denote the graphs $\Gcal'$ for which there exist parameters $F, \Lambda^{(k)}$ defined according to $\Gcal'$ such that the covariance matrix coincides with the covariance matrix of a model with latent DAG $\Gcal$. Are the following containments strict in general
    \[
    \soft(\Gcal) \subset \soft_2(\Gcal) \subset \{\Gcal' \,|\, \overline{\Gcal'} = \overline{\Gcal} \}?
    \]
\end{problem}

Under soft interventions, the space of parameters compatible with a given model is linear and positive dimensional, by Theorem \ref{thm:soft1}. It is then natural to ask for the best solution in this space, for an appropriate notion of best. This would give a choice of unique parameters under soft interventions.

Finally, our assumptions require that \emph{all} the latent error distributions are non-Gaussian. The results may extend to the case where some are Gaussian, cf. \cite{wang2024identifiability}.

On the algorithmic side, there are multiple possible improvements. 
The tensor decomposition contributes significantly to the error in the recovered parameters, see Figure~\ref{fig:recovery_F}. Other tensor decomposition algorithms might give more accurate output. 
One could test our algorithm starting from the factor matrices plus random noise, to study the extent to which our algorithm would work with a sufficiently accurate tensor decomposition. 
Next, one could implement a greedy search over permutations and signs to speed up the recovery of $P^{(k)}$ and $D^{(k)}$. 
Finally, it would be interesting to study the robustness to non-linearity in the latent space (e.g., $Z = (I-\Lambda)^{-1} \epsilon + \alpha \epsilon^2$ for small $\alpha \in \R$) or in the mixing map (e.g., $X = F Z + \alpha Z^2$, where $Z^2$ is a vector with entries $Z_i Z_j$ for all $i,j\in [q]$ and $\alpha \in \R$ small).

\bibliographystyle{alpha}
\bibliography{biblio}

\appendix
\section{Pseudocode}\label{sec:Appendix}
We provide pseudocode for the algorithms in Section \ref{sec:algorithms}. Their implementations are available at \url{https://github.com/paulaleyes14/linear-causal-disentanglement-via-cumulants}.
Below, the $i$-th row of a matrix $M$ is denoted by $\mathbf{m}^i$ and the $i$-th column by $\mathbf{m}_i$.

Algorithms \ref{alg:Lambda} and \ref{alg:F_q<p} below work if the set of intervention targets coincides with the set of latent variables. In theory, this is true by our assumptions. In practice, the algorithm could assign the wrong target to an intervention due to numerical errors. When implementing the algorithm, we force the interventions to be on distinct nodes.

\subsection{General case}

\phantom{a}

\begin{breakablealgorithm} 
\caption{Recovery of the permutation matrix (recover\_perm)}
\label{alg:permutation}
    \begin{algorithmic}[1]
        \State Input: $M = A^{(0)}$ and $M^{(k)} = A^{(k)} D^{(k)} P^{(k)}$.
        \State Output: $P^{(k)}$, the permutation matrix encoding the relabeling of the latent nodes in the context corresponding to an intervention at node $i_k$.
        \vspace{0.5cm}
        \State $q\gets$number of columns of $M$
        \State perm $\gets $ set of all $q\times q$ permutation matrices with entries $\pm 1$
        \State $\sigma \gets $ maximum float 
        \State $P \gets $ None
        \For{ mat in perm}
        \State newmat $\gets M-M^{(k)}\cdot \text{mat}$
        \State $ss \gets $ second largest abs(singular value) of newmat
        \If{$ss<\sigma$}
        \State $\sigma \gets ss$
        \State $P \gets \text{mat}^\top$
        \Else
        \State continue
        \EndIf
        \EndFor
        \State \Return $P$
    \end{algorithmic}
\end{breakablealgorithm}
~\\

\begin{breakablealgorithm} 
\caption{Recovery of the intervention target and scaling (recover\_target\_scaling)}
\label{alg:target}
    \begin{algorithmic}[1]
        \State Input: $M = A^{(0)}$, $M^{(k)} = A^{(k)} D^{(k)} P^{(k)}$ and a threshold \texttt{thr}.
        \State Output: the target of the $k$-th intervention $i_k$ and the diagonal matrix $D^{(k)}$.
        \vspace{0.5cm}
        \State $q\gets $ number of columns of $M$
        \State $D\gets I_{q\times q}$
        \State $P \gets $ recover\_perm$(M,M^{(k)})$
        \State $N \gets M^{(k)} P^\top$
        \State scalings $\gets $ list()
        \State indices $\gets $ list()
        \For{$i = 1$ to $q$}
        \State $v \gets $ project $\mathbf{n}^i$ onto $\mathbf{m}^i$
        \If{$|v - \mathbf{n}^i| < $ \texttt{thr}}
        \State Add $v[1] / \mathbf{m}^i[1]$ to scalings 
        \State Add $i$ to indices 
        \EndIf
        \EndFor
        \State $d\gets $ largest entry of scalings
        \State $i_k \gets $ indices(index of $d$)
        \State $D[i_k,i_k] \gets d$
        \State \Return $i_k$, $D$
    \end{algorithmic}
\end{breakablealgorithm}
~\\

\begin{breakablealgorithm} 
\caption{Recovery of $\Lambda^{(0)}$ (recover\_Lambda)}
\label{alg:Lambda}
    \begin{algorithmic}[1]
        \State Input: $M = A^{(0)}$, $M^{(k)} = A^{(k)} D^{(k)} P^{(k)}$ for all $k\in[q]$, and a threshold \texttt{thr}.
        \State Output: $\Lambda^{(0)}$.
        \vspace{0.5cm}
        \State $q\gets $ number of columns of $M$
        \State $L \gets I_{q\times q}$        
        \For{$k=1$ to $q$}
        \State $P \gets $ recover\_perm$(M,M^{(k)})$
        \State $(i_k, D) \gets $ recover\_target\_scaling$(M,M^{(k)}, \texttt{thr})$
        \State $A = M^{(k)} \; P^\top$ inverse$(D)$
        \For{$j=1$ to $q$}
        \If{$j\neq i_k$}
        \State $L[i_k,j] \gets \frac{M[1,j] - A[1,j]}{M[1,i_k]}$
        \EndIf
        \EndFor
        \EndFor
        \State \Return  $I_{q\times q} - \hbox{inverse}(L)$
    \end{algorithmic}
\end{breakablealgorithm}
~\\

\subsection{Injective case}

\phantom{a}

\begin{breakablealgorithm} 
\caption{Recovery of the intervention target (recover\_target)}
\label{alg:intervention_target_q<p}
\begin{algorithmic}[1]
    \State Input: $C = C^{(0)}$ and $C^{(k)}=(P^{(k)})^\top (D^{(k)})^{-1}(I-\Lambda^{(k)})H$.
    \State Output: ($i_k$, $j_k$) such that $i_k$ is the intervention target of the $k$-th context, and $P^{(k)}_{i_k, j_k}=1$.
    \vspace{0.5cm}
    \State $q \gets $number of rows of $C$
    \State $\text{matched}_\text{obs} \gets$ set$()$
    \State $\text{matched}_\text{int} \gets$ set$()$
    \For{$i=1$ to $q$}
    \If{$\mathbf{c}^i$ has matching row in $C^{(k)}$}
    \State $j\gets$ index of matching row
    \State Add $i$ to $\text{matched}_\text{obs}$
    \State Add $j$ to $\text{matched}_\text{int}$
    \EndIf
    \EndFor
    \State $i_k\gets $ $[q]\backslash \text{matched}_\text{obs}$
    \State $j_k\gets$ $[q]\backslash \text{matched}_\text{int}$
    \State \Return $(i_k,j_k)$
\end{algorithmic}
\end{breakablealgorithm}
~\\

\begin{breakablealgorithm} 
\caption{Recovery of the permutation matrix (recover\_perm)}
\label{alg:permutation_q<p}
    \begin{algorithmic}[1]
        \State Input: $C = C^{(0)}$ and $C^{(k)}$, recovered from the context with intervention target $i_k$.
        \State Output: the permutation matrix $P^{(k)}$.
        \vspace{0.5cm}
        \State $q\gets$number of rows of $C$
        \State $P\gets \mathbf{0}_{q\times q}$
        \State $(i_k,j_k)\gets $ recover\_target($C,C^{(k)}$)
        \State $P[i_k,j_k]\gets1$
        \For{$i=1$ to $q$}
        \If{$i=i_k$}
        \State continue
        \Else
        \State $j\gets$ index of matching row in $C^{(k)}$ to $\mathbf{c}^i$
        \State $P[i,j]\gets 1$
        \EndIf
        \EndFor
        \State \Return $P$
    \end{algorithmic}
\end{breakablealgorithm}
~\\

\begin{breakablealgorithm} 
\caption{Recovery of the scaling (recover\_scaling)}
\label{alg:scaling_q<p}
    \begin{algorithmic}[1]
        \State Input: $C = C^{(0)}$ and $C^{(k)}$.
        \State Output: the diagonal matrix $D^{(k)}$.
        \vspace{0.5cm}
        \State $q\gets$number of rows of $C$
        \State $D\gets I_{q\times q}$
        \State $(i_k,j_k)\gets $ recover\_target($C,C^{(k)}$)
        \State $P\gets $ recover\_perm($C,C^{(k)}$)
        \State $B^{(0)} \gets $ pseudoinverse($P C^{(0)}$)
        \State $B^{(k)} \gets $ pseudoinverse($P C^{(k)}$)
        \State $v \gets $ project $\mathbf{b^{(k)}}_{i_k}$ onto $\mathbf{b^{(0)}}_{i_k} $
        \State $D[i_k,i_k] \gets v[1]/\mathbf{b^{(0)}}_{i_k}[1]$ 
        \State \Return $D$
    \end{algorithmic}
\end{breakablealgorithm}
~\\

\begin{breakablealgorithm} 
\caption{Recovery of $F = H^+$ (recover\_F)}
\label{alg:F_q<p}
    \begin{algorithmic}[1]
        \State Input: $C = C^{(0)}$ and $C^{(k)}$ for all $k\in [q]$.
        \State Output: $F$ such that $C^{(0)}=(I-\Lambda^{(0)})H$ and $H = F^+$.
        \vspace{0.5cm}
        \State $q\gets$ number of rows in $C$
        \For{$k=1$ to $q$}
        \State $(i_k,j_k)\gets $ recover\_target($C,C^{(k)}$)
        \State $P\gets$ recover\_perm($C,C^{(k)}$)
        \State $D\gets$ recover\_scaling($C,C^{(k)}$)
        \State $A = P D C^{(k)}$
        \State $\mathbf{h}^{i_k} = \mathbf{a}^{i_k}$
        \EndFor
        \State \Return pseudoinverse($H$)
    \end{algorithmic}
\end{breakablealgorithm}
~\\

\end{document}